\documentclass{article}

	\usepackage[preprint, nonatbib]{neurips_2024}

\usepackage[utf8]{inputenc} 
\usepackage[T1]{fontenc}    
\usepackage{hyperref}       
\usepackage{url}            
\usepackage{booktabs}       
\usepackage{amsfonts}       
\usepackage{nicefrac}       
\usepackage{microtype}      
\usepackage[dvipsnames]{xcolor}         
\usepackage{algorithm}
\usepackage{algpseudocode}
\usepackage{graphicx}
\usepackage{multirow}
\usepackage{soul}
\usepackage{xspace}
\usepackage{comment}
\usepackage{caption}
\usepackage{subcaption}
\usepackage{enumitem}

\newcommand{\NAME}{\texttt{DCFL}\xspace}

\title{Using Diffusion Models as Generative Replay in Continual Federated Learning -- What will Happen?}

\vspace{-5pt}
\author{
  \vspace{-25pt}\\
  \textbf{Yongsheng Mei$^{1}$\thanks{Equal contribution} \quad Liangqi Yuan$^{2\,*}$} \\ \textbf{\quad Dong-Jun Han$^{3}$ \quad Kevin S. Chan$^{4}$ \quad Christopher G. Brinton$^{2}$ \quad Tian Lan$^{1}$}\vspace{3pt} \\
  $^1$The George Washington University ~~\quad\quad $^2$Purdue University \\ $^3$Yonsei University ~~\quad\quad $^4$Army Research Lab\vspace{3pt} \\
  \texttt{\small ysmei@email.gwu.edu, 
liangqiy@purdue.edu,} \\ 
  \texttt{\small 
djh@yonsei.ac.kr, kevin.s.chan.civ@army.mil, cgb@purdue.edu, 
tlan@gwu.edu}\vspace{8pt}  \\
  Code:~\, \url{https://github.com/ysmei97/DCFL}
  \vspace{-4pt}
}

\usepackage{amsmath,amsfonts,bm,amssymb,amsthm}

\def\eqref#1{equation~(\ref{#1})}
\def\Eqref#1{Equation~(\ref{#1})}

\def\1{\bm{1}}

\def\vy{{\bm{y}}}
\def\vz{{\bm{z}}}

\def\mI{{\bm{I}}}

\DeclareMathAlphabet{\mathsfit}{\encodingdefault}{\sfdefault}{m}{sl}
\SetMathAlphabet{\mathsfit}{bold}{\encodingdefault}{\sfdefault}{bx}{n}

\def\gD{{\mathcal{D}}}

\def\gG{{\mathcal{G}}}

\def\gL{{\mathcal{L}}}

\def\gN{{\mathcal{N}}}

\def\gU{{\mathcal{U}}}

\def\gX{{\mathcal{X}}}

\newcommand{\E}{\mathbb{E}}

\newtheorem{theorem}{Theorem}

\newtheorem{lemma}{Lemma}

\newtheorem{assumption}{Assumption}

\begin{document}
	
    \maketitle

    \begin{abstract}
    Federated learning (FL) has become a cornerstone in decentralized learning, where, in many scenarios, the incoming data distribution will change dynamically over time, introducing continuous learning (CL) problems. This continual federated learning (CFL) task presents unique challenges, particularly regarding catastrophic forgetting and non-IID input data. Existing solutions include using a replay buffer to store historical data or leveraging generative adversarial networks. Nevertheless, motivated by recent advancements in the diffusion model for generative tasks, this paper introduces \NAME, a novel framework tailored to address the challenges of CFL in dynamic distributed learning environments. Our approach harnesses the power of the conditional diffusion model to generate synthetic historical data at each local device during communication, effectively mitigating latent shifts in dynamic data distribution inputs. We provide the convergence bound for the proposed CFL framework and demonstrate its promising performance across multiple datasets, showcasing its effectiveness in tackling the complexities of CFL tasks.
    \end{abstract}

    \section{Introduction}

Federated learning (FL) has emerged as a prevalent decentralized learning paradigm, allowing training of a global model through interactions with distributed clients while maintaining the privacy of their local data \cite{mcmahan2017communication,kairouz2021advances}. Most FL frameworks operate under the assumption that the client datasets at each client remain static throughout extensive learning cycles and iterations. However, this assumption does not align with the dynamic nature of real-world scenarios \cite{ding2022federated,chahoud2023demand}. The global model trained on such fixed datasets often fails to adapt effectively to the constantly evolving real world \cite{kim2021dynamic}. Furthermore, in real-world scenarios, clients often encounter new environments, objectives, and tasks -- an aspect of adaptability that conventional FL frameworks have not yet fully addressed.

Continual learning (CL) methods were proposed to handle the phenomenon of catastrophic forgetting, where historical data may become inaccessible due to privacy regulations or storage constraints \cite{wang2024comprehensive}. These methods were proposed that focus on developing systems capable of continuously learning from new tasks without erasing previously acquired knowledge. Classical CL scenarios can be broadly categorized into three types, ranging from task incremental learning (TIL) and domain incremental learning (DIL) to class incremental learning (CIL) \cite{van2022three}. However, these CL scenarios may face broader and more diverse challenges in the context of FL, as it is essential to consider cross-client scenarios and the non-IID data distribution among clients.

We aim to address the challenges of continual federated learning (CFL) tasks in practical settings. Specifically, in CFL, learning is decentralized across multiple heterogeneous devices and is coordinated by a central server, where devices encounter new data and tasks over time. This poses challenges in handling both catastrophic forgetting issues induced by timely-shifted data distribution and non-IID problems in FL \cite{yuan2023peer}. Recent approaches have been proposed to mitigate these challenges, such as leveraging replay memory to store historical data experienced by the model in the past \cite{yoon2021federated} and utilizing generative adversarial networks (GANs) to generate historical data on each device to help remember past experiences \cite{qi2022better}. Some methods utilize the global model on the server to train a generative model through knowledge distillation, but this approach leads to low-quality synthetic data and introduces additional noise \cite{babakniya2024data,zhang2023target}. While storing real historical data \cite{dong2022federated} is useful for memory replay, it may not be feasible in cases where the data is available only for limited-time usage. Alternatively, FOT \cite{bakman2023federated} performs global subspace extraction to identify features of previous tasks, aiming to prevent forgetting. However, FOT incurs higher communication costs between clients and servers due to the transfer of subspace information and orthogonal projections. GANs, as traditional generative models, on the other hand, involve learning two models (generator and discriminator) to reach a stable equilibrium, which can be difficult to train and sometimes susceptible to mode collapse \cite{srivastava2017veegan} problems. Therefore, the powerful data generation capability demonstrated by the diffusion model \cite{ho2020denoising} in various domains \cite{amit2021segdiff,austin2021structured,avrahami2022blended} makes it a strong candidate to replay data within the CFL context.

In this paper, we propose a novel framework \NAME that integrates CFL with conditional diffusion. At each local device, the embedded diffusion model serves to alleviate the impact of catastrophic forgetting by generating synthetic historical data. Since the diffusion model is not shared with anyone, our framework adheres to general privacy restrictions. Subsequently, the target models (i.e., models performing FL) are aggregated on the global server to obtain a generalized global model. We also provide a convergence analysis of \NAME by separately examining the convergence of the FL backbone, the data distribution shift, and the data generation convergence with the diffusion model. By combining these results, we demonstrate that the overall convergence of the system ultimately hinges on the performance of the introduced diffusion model, of which the bounded characteristic contributes to the system's convergence. \NAME has been tested on three CFL scenarios and four mainstream benchmark datasets, where \NAME significantly outperformed classical FL, classical CL, traditional generative model, and state-of-the-art (SOTA) baselines.

The main contributions of this paper are provided as follows:
\begin{itemize}
    \vspace{-.5em}
    \item We introduce a novel CFL framework, termed \NAME, which eliminates the need for replay memory, enabling model learning for both local clients and the global server with dynamic data inputs. \NAME leverages the modern diffusion model to generate synthetic historical data based on previously observed data distributions \textbf{(Section \ref{sec:CFL-CD})}.
    \vspace{-.1em}
    \item We provide the convergence analysis for our \NAME framework. Our convergence result captures the bound of the FL model, the bound affected by the diffusion model, and the effect of data distribution shift between time steps \textbf{(Section \ref{sec:proof})}. 
   \vspace{-.1em}
   \item  We conduct extensive experiments using MNIST, FashionMNIST, CIFAR-10, and PACS datasets under three practical CFL environments. The results demonstrate that our \NAME framework improves upon the best baseline by $32.61\%$ in the Class Incremental IID scenario, $15.16\%$ in the Class Incremental Non-IID scenario, and $7.45\%$ in the Domain Incremental scenario \textbf{(Section \ref{sec:experiments}).}
    \vspace{-.3em}
\end{itemize}
To the best of our knowledge, our \NAME is the first work that successfully integrates diffusion models into continual federated learning, addressing its unique challenges with theoretical analysis. 
    \section{Preliminary}
\vspace{-1mm}

\subsection{Continual Federated Learning Scenarios}
\label{sec:scenarios}
\vspace{-1mm}

Unlike classical CL setups, CFL tends toward greater diversity due to the presence of multiple clients. In the FL context, it requires (i) all clients to have the same model architecture, (ii) uniform consensus among clients (i.e., all clients agree on the definitions of labels), and (iii) the same task (i.e., the same global test set). Additionally, FL typically needs to address non-IID settings, where clients have different class distributions. Therefore, CFL introduces entirely distinct scenario setups, and so far there has been no unified approach in the literature.

We introduce three different CFL scenarios: 1) Class Incremental IID, 2) Class Incremental Non-IID, and 3) Domain Incremental settings, as shown in Figure \ref{Fig. FL_CL_Scenarios}. It is important to note that the IID setting here differs from the traditional IID setting, such as those considered in \cite{mcmahan2017communication}. In conventional FL, the IID setting is modeled by distributing the whole training set uniformly at random across all clients, without considering dataset evolution over time. In contrast, our IID setting in the CFL setting ensures that all clients have identical class distribution at any given time but only with a subset of classes (e.g., only labels $0$ and $1$). We model the dataset evolution by letting the clients have different set of classes (e.g., labels $2$ and $3$) in the next time step.

\begin{figure*}[t]
    \centering
    \includegraphics[width=1.\textwidth]{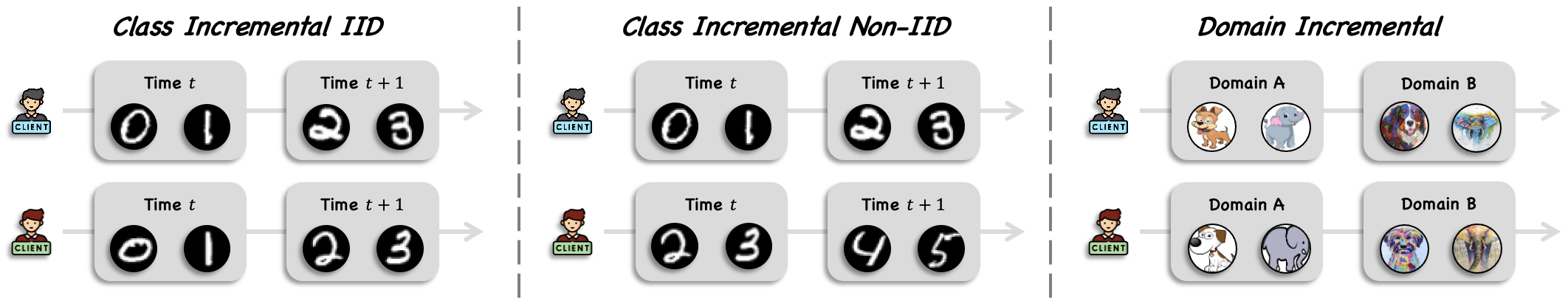}
    \caption{\textbf{Three Continual Federated Learning Scenarios.} Class Incremental IID: Clients have an identical class distribution, with classes incrementing over time. Class Incremental Non-IID: Clients have a non-identical class distribution, with classes incrementing over time. Domain Incremental: Clients data domain changes over time.}
    \label{Fig. FL_CL_Scenarios}
\end{figure*}

\subsection{Federated Averaging}

We take the vanilla FedAvg \cite{mcmahan2017communication} model as the general framework of our design, where the objective function with model parameter $\theta$ is defined as:
\begin{equation}
    \begin{aligned}
        &\min_\theta \left\{ F(\theta) \triangleq \sum_{k=1}^{K} p_k F_k(\theta) \right\}, \\
        &\mathrm{s.t.} \quad \sum_{k=1}^{K} p_k = 1, \quad p_k \ge 0,
    \end{aligned}
\end{equation}
where $K$ is the number of clients and $p_k$ is the weight of the $k$-th client. The local model learning loss function $F_k(\theta)$ is given by:
\begin{equation}
	F_k(\theta) \triangleq \frac{1}{a_k} \sum_{j=1}^{a_k} f(\theta;X_{k,j}) \quad \text{where} \quad a_k = |X_{k}|,
\end{equation}
where $f(\cdot)$ is the loss function, and $a_k$ is the number of training data in $k$-th client. Considering local clients in the $t$-th round, we have the aggregation of the global model as:
\begin{equation}
    \theta_{t} \leftarrow \sum^{K}_{k=1} p_k \theta^k_{t} \quad \text{where} \quad p_k = \frac{|X_{k}|}{|X|},
    \label{eq:aggregation}
\end{equation}
where $n$ represents the total number of samples across all clients. As in the FedAvg setup, the aggregation of models is performed by weighting each client's contribution according to the number of samples they possess.

\subsection{Denoising Diffusion Probabilistic Models}

Basic DDPM \cite{ho2020denoising} gradually adds random noise to the data over a series of time steps $(x_0, \cdots, x_N)$ in the forward process, where $x_0 = x,\ x_N = z$. Specifically, the sample at each time step is sampled from a Gaussian distribution conditioned on the sample from the previous time step with predefined schedule $\beta_{1:N}$,
\begin{equation}
	x_n \sim F(\cdot|x_{n-1})=\gN(\sqrt{1-\beta_n}x_{n-1}, \beta_n \mI).
	\label{eq:ddpm_forward}
\end{equation}
With \eqref{eq:ddpm_forward}, the sample at each step $t$ can be expressed as a function of $x_0$: $x_n = \sqrt{\alpha_n}x_0 + \sqrt{1 - \alpha_n} \epsilon$, where $\alpha_n = \prod_{s=0}^n 
(1 - \beta_s),\ \epsilon \sim \gN(0, \mI)$ \cite{sohl2015deep}. On the contrary, $x_N$ can be converted back to $x_0$ step-wisely via the reverse denoising process:
\begin{equation}
	x_{n-1} \sim G_{\omega}(\cdot|x_{n}) = \gN(\mu_\omega(x_n, n),\Sigma_\omega(x_n, n)),
\end{equation}
where $\mu_{\omega}$ and $\Sigma_\omega$ can be obtained from neural networks. The objective of learning this denoising function is to match the joint distributions of $x_{0:N}$ in the forward and reverse processes. To optimize this objective, Ho \textit{et al.} proposed a reformulation \cite{ho2020denoising} by specifying the variance schedule $\beta_{1:N}$ and fixing the reverse variance $\Sigma_\omega(x_n, n)$ to be $\beta_n \mI$, which is:
\begin{equation}
    \begin{aligned}
        \min_{\omega} \{&\gL(\omega) \triangleq \mathbb{E}_{n \sim \gU[1, N], x_0 \sim P_X(\cdot), \epsilon \sim \gN(0, \mI)} [ \lambda(n) \\
        &\Vert\epsilon - \epsilon_{\omega}(\sqrt{\alpha_n}x_0 + \sqrt{1-\alpha_n}\epsilon, n)\Vert^2] \},
    \end{aligned}
\label{ddpm:obj_real}
\end{equation}
where $\gL(\cdot)$ is the loss function of the diffusion model and $\lambda(n)$ is a positive weighting function usually set as 1 for all $n$ to improve sample quality. Recently, the polynomial bounds on the convergence rate of the diffusion model have also been given in \cite{benton2023linear}.
    \section{Methodology}
\label{sec:methodology}

\subsection{Proposed \NAME Framework}
\label{sec:CFL-CD}

We propose using a conditional diffusion model for replay in CFL, termed \NAME, as depicted in Figure \ref{Fig. FL_Diffusion}. In the \NAME framework, each client possesses a target model $\theta$ (for various FL tasks) and a diffusion model $\omega$ for replaying data distributions from previous time periods. The server receives and aggregates only the target models from each client while remaining unaware of the clients' diffusion models.

\begin{figure*}[ht]
    \centering
    \includegraphics[width=0.9\textwidth]{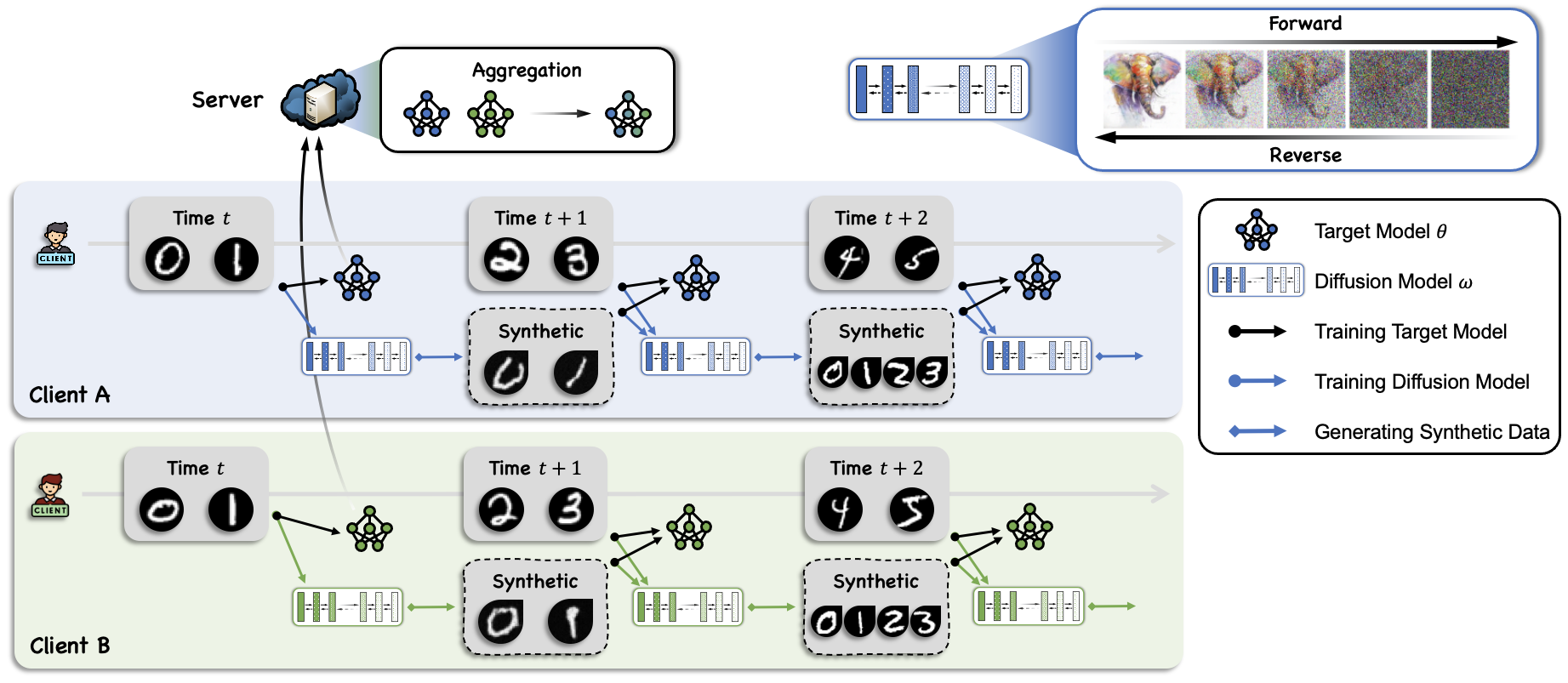}
    \caption{\textbf{Proposed \NAME Framework.} Each client has a target model and a diffusion model, both trained on the same dataset, consisting of the previous time period's real and synthetic data. The target model is uploaded to the server for aggregation, while the diffusion model remains local to prevent privacy leakage. The trained diffusion model will generate synthetic data encompassing all previously acquired knowledge.} 
    \label{Fig. FL_Diffusion}
\end{figure*}

We utilize Algorithm \ref{alg:cfl} to outline the general workflow of \NAME. The complete version is provided in Appendix \ref{appendix:alg}. As we adopt FedAvg as the backbone for FL, both the local devices and the global server perform updates and aggregation in a manner similar to FedAvg, including training the target local model with gradient descent, as shown in the algorithm. Additionally, we incorporate the diffusion model (in line 3 of Algorithm \ref{alg:cfl}) for each local device to address the latent input data distribution shift. This diffusion model is utilized to recover historical data experienced by generating synthetic data, thereby mitigating the issue of catastrophic forgetting. Before training the local target model, the diffusion model generates a portion of historical synthetic data mixed with current real data as input for learning. The diffusion model is also trained on the mixed dataset of real and synthetic data, which helps prevent catastrophic forgetting within the diffusion model itself. This process is repeated for each local client, and the learned model gradients are aggregated at the end of every communication round.

\begin{algorithm}[t]
\footnotesize
	\caption{\textbf{Proposed \NAME Framework} - See Algorithm \ref{alg:cfl_comp} for Complete Procedure}
	\label{alg:cfl}

        \textbf{Input:} Communication rounds ($T$), client datasets ($\gD^k_t$), \textcolor{Tan}{target model, loss function, and learning rate ($\theta$, $F$, $\eta_\theta$)}, \textcolor{RoyalBlue}{diffusion model, loss function, and learning rate ($\omega$, $\gL$, $\eta_\omega$)}

        \textbf{Output:} Generalized global model ($\theta_T$)
        \vspace{5pt}

	\begin{algorithmic}[1]
		\Statex \textbf{Local update} of the $k$-th client:
        \State \textbf{initialize} $\omega_0$
		\For{each round $t=1:T$}
        \State Obtain $\gG^k_{t-1} \leftarrow \omega^k_{t-1}(\gG^k_{t-1,n}\sim\gN(0,\mI))$ via reverse process \textcolor{RoyalBlue}{\Comment{Generate synthetic data}}
        \State Combine real and synthetic data with a scale factor $\delta$: $\gX^k_t=\gD^k_t \cup \delta\cdot\gG^k_{t-1}$

        \State $\theta^k_{t} \leftarrow \theta^k_{t-1} - \eta_\theta\nabla F_k(\theta^k_{t-1};\gX^k_t)$ \textcolor{Tan}{\Comment{Train target model}}
        \State {\bf repeat}\xspace$\omega^k_{t} \leftarrow \omega^k_{t-1} - \eta_\omega\nabla \gL_k(\omega^k_{t-1};\gX^k_t)$\hspace{0.5em}{\bf until}\xspace converge 
            \Statex \textcolor{RoyalBlue}{\Comment{Train diffusion model}}
		\EndFor
		\State \textbf{return} $\theta^k_{t}$ to server
		\Statex

\end{algorithmic}
\begin{algorithmic}[1]

		\Statex \textbf{Global update} of the server
		\State \textbf{initialize} $\theta_0$
		\For{each round $t=1:T$}
		\For{each client $k=1:K$ \textbf{in parallel}}
		\State $\theta^k_{t} \leftarrow$ $k$-th client's \textit{local update}
		\EndFor
		\State $\theta_t \leftarrow \sum_{k=1}^{K}p_k\theta^k_t$ \textcolor{Tan}{\Comment{Aggregate target models}}
		\EndFor
	\end{algorithmic}
\end{algorithm}

\subsection{Convergence Bound of CFL Framework with Diffusion}
\label{sec:proof}

To determine the convergence bound of the designed CFL model, we need to 1) find the convergence of the general FL framework with both original data and synthetic data, 2) verify the convergence of data generation in the integrated diffusion model, and 3) prove the convergence of the whole model regarding the latent input data distribution shift. To start with, we provide several standard assumptions that have been widely used in the FL literature \cite{li2019convergence, reisizadeh2020fedpaq}.

\begin{assumption} 
	$F_1, \cdots, F_K$ are all $L$-smooth:
	for all model parameters $\theta_1$ and $\theta_2$, $F_k(\theta_1)  \leq F_k(\theta_2) + (\theta_1 - \theta_2)^T \nabla F_k(\theta_2) + \frac{L}{2} \Vert \theta_1 - \theta_2\Vert_2^2$.
	\label{assumption:smooth}
\end{assumption}

\begin{assumption} 
	$F_1, \cdots, F_K$ are all $\mu$-strongly convex:
	for all model parameters $\theta_1$ and $\theta_2$, $F_k(\theta_1)  \geq F_k(\theta_2) + (\theta_1 - \theta_2)^T \nabla F_k(\theta_2) + \frac{\mu }{2} \Vert \theta_1 - \theta_2\Vert_2^2$.
	\label{assumption:strong_cvx}
\end{assumption}

\begin{assumption} 
	Let $\xi_t^k$ be sampled from the $k$-th device's local data uniformly at random.
	The variance of stochastic gradients in each device is bounded: $\E \left\Vert \nabla F_k(\theta_t^k,\xi_t^k) - \nabla F_k(\theta_t^k) \right\Vert^2 \le \sigma_k^2$ for $k=1,\cdots,K$.
	\label{assumption:sgd_var}
\end{assumption}

\begin{assumption} 
	The expected squared norm of stochastic gradients is uniformly bounded, i.e., $\E \left\Vert \nabla F_k(\theta_t^k,\xi_t^k) \right\Vert^2  \le G^2$ for all $k=1,\cdots,K$ and $t=1,\cdots, T-1$.
	\label{assumption:sgd_norm}
\end{assumption}

The basic FedAvg model has been proven to converge to a global optimum in non-iid settings \cite{li2019convergence}. When all the devices participate in the aggregation step and the FedAvg algorithm terminates after $T$ rounds, the following lemma will hold.
\begin{lemma}[FedAvg convergence bound \cite{li2019convergence}]
	Let Assumptions~\ref{assumption:smooth} to \ref{assumption:sgd_norm} hold and $L, \mu, \sigma_k, G$ be defined therein. Choose $\kappa = \frac{L}{\mu}$, $\gamma = \max\{8\kappa, E\}$ and the learning rate $\eta_t = \frac{2}{\mu (\gamma+t)}$. 
	Then FedAvg with full device participation to the optimal $F^*$ satisfies:
	\begin{equation}
		\E \left[ F(\theta_T)\right] - F^* \leq \frac{\kappa}{\gamma +T-1} \left( \frac{2B}{\mu} + \frac{\mu \gamma}{2} \E \Vert\theta_1 - \theta^*\Vert^2 \right),
		\label{eq:bound_K=K}
	\end{equation}
	where $B = \sum_{k=1}^K p_k^2 \sigma_k^2 + 6L \Gamma + 8  (E-1)^2G^2$ and $\Gamma = F^* - \sum^K_{k=1}p_k F^*_k$ measuring the degree of non-iid.
	\label{lemma:fedavg_bound}
\end{lemma}

\begin{proof}
	See Section 3.2 in \cite{li2019convergence}.
\end{proof}

Different from conventional FL, in the CL context, the data distribution will shift at every step, introducing the catastrophic forgetting issue. In our framework, we avoid this problem by leveraging the diffusion model to generate labeled synthetic data that can reflect the experienced historical real data without storing them physically. Therefore, at step $t + 1$, the input data $\gX_{t+1}$ is the combination of real data $\gD_{t+1}$ and synthetic data $\gG_t$ from the diffusion model recovering the data from the previous step, satisfying:
\begin{equation}
    \gX_{t+1} \equiv \gD_{t+1} \cup \delta\cdot\gG_t,
    \label{eq:relation}
\end{equation}
where $\delta$ is an extra scale factor controlling the amount of generated synthetic data compared to real data to mitigate the negative impacts of inaccurate synthesis generations that cannot represent previous real inputs. For simplicity of the derivation, we omit this factor (set $\delta = 1$) in this proof, while we use the sensitivity study regarding $\delta$ in Appendix \ref{appendix:Sensitivity NumSynSamples} for a thorough discussion. Besides, according to \eqref{eq:relation}, the sampled data distributions ${\xi}^k_{t+1}$ regarding each part can be describe as follows:
\begin{equation}
	\begin{aligned}
		&\xi^k_{t+1} \sim \gX_{t+1}, \quad \tilde{\xi}^k_{t+1} \sim \gD_{t+1}, \quad \tilde{\tilde{\xi}}^k_{t+1} \sim \gG_t,
	\end{aligned}
        \label{eq:distribution}
\end{equation}

From $t$ to $t + 1$, the input data distribution is different as new data loaded at $t + 1$ has not been seen at $t$, and the synthetic data recovering the loaded data at $t$ will deviate from the real data distribution. We capture such data distribution shift in a controllable range, measured by $\Delta_t$, which aligns with most learning scenarios. This character is depicted in the following assumption.

\begin{assumption}
    In each continual learning step, the incoming data distribution shift is bounded and deviation of can be captured in a measurable range with $\Delta_t$, which is: $D_\mathrm{KL}(F^*_t(\theta^k_t;\xi^k_t) \parallel F^*_{t+1}(\theta^k_{t+1};\tilde{\xi}^k_{t+1})) \leq \Delta_t$.
    \label{assumption:shift_bound}
\end{assumption}

Based on the given assumption, we measure the distance between $F^*_t$ and $F^*_{t+1}$ with KL divergence:

\begin{theorem}[Data distribution deviation measurement]
    The following equation can further bound the KL divergence:
    \begin{equation}
        \begin{aligned}
            &D_\mathrm{KL}\left( F^*_t(\theta^k_t;\xi^k_t) \parallel F^*_{t+1}(\theta^k_{t+1};\xi^k_{t+1}) \right) \\
            &\leq \frac{1}{2}\left[ D_\mathrm{KL}\left( F^*_t(\theta^k_t;\xi^k_t) \parallel F^*_{t+1}(\theta^k_{t+1};\tilde{\tilde{\xi}}^k_{t+1}) \right) + \Delta_t \right],
        \end{aligned}
        \label{eq:data_bound}
        \end{equation}
    where $\Delta_t$ is defined by the upper bound of incoming data distribution shift based on Assumption~\ref{assumption:shift_bound}.
    \label{theorem:data_bound}
\end{theorem}

\begin{proof}
    See Appendix \ref{proof:data_bound}.
\end{proof}

The KL divergence in the right-hand side of \eqref{eq:data_bound} indicates the distance between original real data from step $t$ and generated synthetic data in step $t + 1$. We use the diffusion model as the generative model in our framework, where we aim to reconstruct the real data distribution from the previous step with newly generated synthetic data. Therefore, it is equivalent to measuring the reconstruction convergence of the incorporated diffusion model. Recently, such a convergence bound has been proposed in~\cite{benton2023linear}, summarized in the following lemma.

\begin{lemma}[Convergence of data generation via the diffusion model) \cite{benton2023linear}]
    Suppose there exists $\kappa > 0$ controlling the diffusion step size $\gamma$ such that for each $m=0,\dots, M-1$, we have $\gamma_m \leq \kappa \in\{1, N - n_{m+1}\}$ where $N$ is the total step in the diffusion model. Then, the bound of the approximated reverse process for data generation is:
    \begin{equation}
        \begin{aligned}
            &D_\mathrm{KL}\left( F^*_t(\theta^k_t;\xi^k_t) \parallel F^*_{t+1}(\theta^k_{t+1};\tilde{\tilde{\xi}}^k_{t+1}) \right) \\
            &\lesssim \epsilon^2_{\rm score} + \kappa^2 dM + \kappa dN + d\exp(-2N),
        \end{aligned}
    \end{equation}
    where $d$ is the dimension of the data and $\epsilon_{\rm score}$ denotes the maximum score approximation {\rm \cite{chen2023score}} error.
    \label{lemma:diffusion_bound}
\end{lemma}
\begin{proof}
	See Section 3 in~\cite{benton2023linear}.
\end{proof}

Therefore, by integrating Lemma \ref{lemma:fedavg_bound},  \ref{lemma:diffusion_bound} and Theorem \ref{lemma:fedavg_bound}, our CFL framework with the diffusion model \NAME can be proven bounded during the learning.

\begin{theorem}[Convergence of \NAME]
    The convergence bound derived from Lemma \ref{lemma:fedavg_bound},  \ref{lemma:diffusion_bound} and Theorem \ref{theorem:data_bound} is:
    \begin{equation}
		\begin{aligned}
            \E \left[ F(\theta_T)\right] - F^*_{T} \leq &\frac{\kappa}{\gamma +T-1} \left( \frac{2B}{\mu} + \frac{\mu \gamma}{2} \E \Vert\theta_1 - \theta^*\Vert^2 \right) \\ 
            & + (1 - 2^{-T}) (\epsilon^2_{\rm score} + \kappa^2 dM + \kappa dN  \\
            & + d\exp(-2N) ) + 2^{-T}\Delta_T.
		\end{aligned}
        \label{eq:cfl_bound}
	\end{equation}
    \label{theorem:cfl_bound}
\end{theorem}

\begin{proof}
    See Appendix \ref{proof:cfl_bound}.
\end{proof}

The bound calculated in \eqref{eq:cfl_bound} offers a theoretical measurable convergence bound for our proposed continual federated learning model, which utilizes the diffusion model as a synthetic data generator. The convergence bound, as outlined in Theorem \ref{theorem:cfl_bound}, comprises three components: the convergence bound of the federated learning model, the bound of the diffusion model, and the divergence of the data distribution. As the number of communication rounds approaches infinity ($T\rightarrow\infty$), the first and third terms in \eqref{eq:cfl_bound} tend towards zero, rendering only the second term relevant. This suggests that the convergence ultimately hinges on the performance of the introduced diffusion model. However, according to Corollary 1 in \cite{benton2023linear}, under specific conditions, the second term relies solely on the error $\epsilon^2_{\rm score}$, signifying the final convergence of the entire system.

It is worth noting that the model experiences an inevitable distribution shift at each step. However, the last term in \eqref{eq:cfl_bound} decreases monotonically, indicating that the deviation $2^{-t}\Delta_t$ regarding a particular round $t$ is mitigated with each training step, as the diffusion model will contribute synthetic data for alleviating the data distribution shift.

    \section{Experiments}
\label{sec:experiments}

\subsection{Experimental Setup}
\label{sec:Experimental Setup}

\textbf{Datasets.} 
We adopt commonly used datasets, including \textbf{MNIST} \cite{lecun1998gradient}, \textbf{Fashion-MNIST} \cite{xiao2017fashion}, and \textbf{CIFAR-10} \cite{krizhevsky2009learning}, for the Class Incremental IID and Class Incremental Non-IID CFL scenarios described in Figure \ref{Fig. FL_CL_Scenarios}. Our sampling method aligns with FedAvg's but is adapted for the CFL setting. Specifically, we uniformly split the datasets into 200 \textit{shards} based on their classes, with each client accessing only 2 shards during any given \textit{session} (i.e., a period of time). The distribution of data changes in subsequent sessions according to the different CFL scenario settings. We consider $T=100$ communication rounds for preliminary experiments, with the clients' data distribution changing every 20 rounds, resulting in 5 sessions ($S = 100/20 = 5$). Additionally, we use the popular domain generalization dataset \textbf{PACS} \cite{li2017deeper} for the Domain Incremental CFL scenario. We consider each client to have all classes within any given domain, and the clients' data changes across the 4 domains in the sequence Sketch $\rightarrow$ Cartoon $\rightarrow$ Art Painting $\rightarrow$ Photo (increasing the level of realism over time) \cite{zhao2023does}. Details of the datasets, scenario settings, and data preprocessing can be found in Appendix \ref{appendix:Datasets}.

\textbf{Implementation.} To reduce computational overhead, the target model is trained 20 times during any session, whereas the diffusion model is trained only once and generates synthetic data once. Apart from $T=100$ communication rounds, we set the target model to undergo $E_\theta=5$ local epochs, while the diffusion model is trained for $E_\omega=100$ or $E_\omega=1000$ local epochs. Both models employ the Adam optimizer with a learning rate of 1e-4. We employ the same CNN architecture used in FedAvg for constructing the target model, consisting of two convolutional layers and two fully connected layers. 
The backbone of the diffusion model is a conditional UNet, structured with one convolutional layer and four-channel layers scaled to 64, 128, 128, and 256, respectively. Real labels (and the domain for PACS) are utilized as conditions to guide data synthesis by the diffusion model. We set the diffusion model to generate a number of synthetic samples equal to the number of real samples, the ratio $\delta=1$. Details of implementation details can be found in Appendix \ref{appendix:Implementation}.

\textbf{Baselines.} 
We compare our \NAME with baselines with several frameworks: FL algorithms, FL with classical CL methods, and state-of-the-art (SOTA) CFL frameworks. Specifically, for FL, we compare with \textbf{FedAvg} \cite{mcmahan2017communication} and \textbf{FedProx} \cite{li2020federated}. For the integration of FL and CL, we implement \textbf{FedAvg+LwF} \cite{li2017learning} and \textbf{FedAvg+EWC} \cite{kirkpatrick2017overcoming}. We adopt \textbf{FedAvg+ACGAN} \cite{odena2017conditional} as the integration of the FL and generative model. For SOTA CFL frameworks, we consider \textbf{FedCIL} \cite{qi2023better}, \textbf{FOT} \cite{bakman2023federated}, \textbf{MFCL} \cite{babakniya2024data}, and \textbf{TARGET} \cite{zhang2023target}. Details of the baseline algorithms and settings can be found in Appendix \ref{appendix:Baselines}.

\subsection{Main Experimental Results}

Figure \ref{Fig. Experimental results} illustrates the performance comparison of \NAME against baselines across three CFL scenarios on three datasets. The results indicate that catastrophic forgetting significantly impacts the clients' ability to retain prior knowledge. In the Class-Incremental IID scenario, the class distributions among all clients change every $T=20$ rounds, with each client generally aware of a subset (i.e., 20\%) of all classes at any given time. The stepwise improvement in accuracy suggests effective retention and mastery of both current and previous knowledge by the model. For the Class Incremental Non-IID scenario, all clients are aware of all class information at any time, leading to convergent behavior across all frameworks. However, the replay functionality in our framework, which ensures that each client's target model is trained with multiple class information, facilitates faster convergence. 

\begin{figure}[ht]
\scriptsize
\centering
\setlength{\tabcolsep}{5pt}
\renewcommand{\arraystretch}{0.8}
\begin{tabular}{cccc}
& \textbf{\hspace{3pt} MNIST} & \textbf{\hspace{3pt} Fashion-MNIST} & \textbf{\hspace{3pt} CIFAR-10} \\
\rotatebox[origin=c]{90}{\textbf{Class Incremental IID}} &
\raisebox{-0.5\height}{\includegraphics[width=0.3\linewidth]{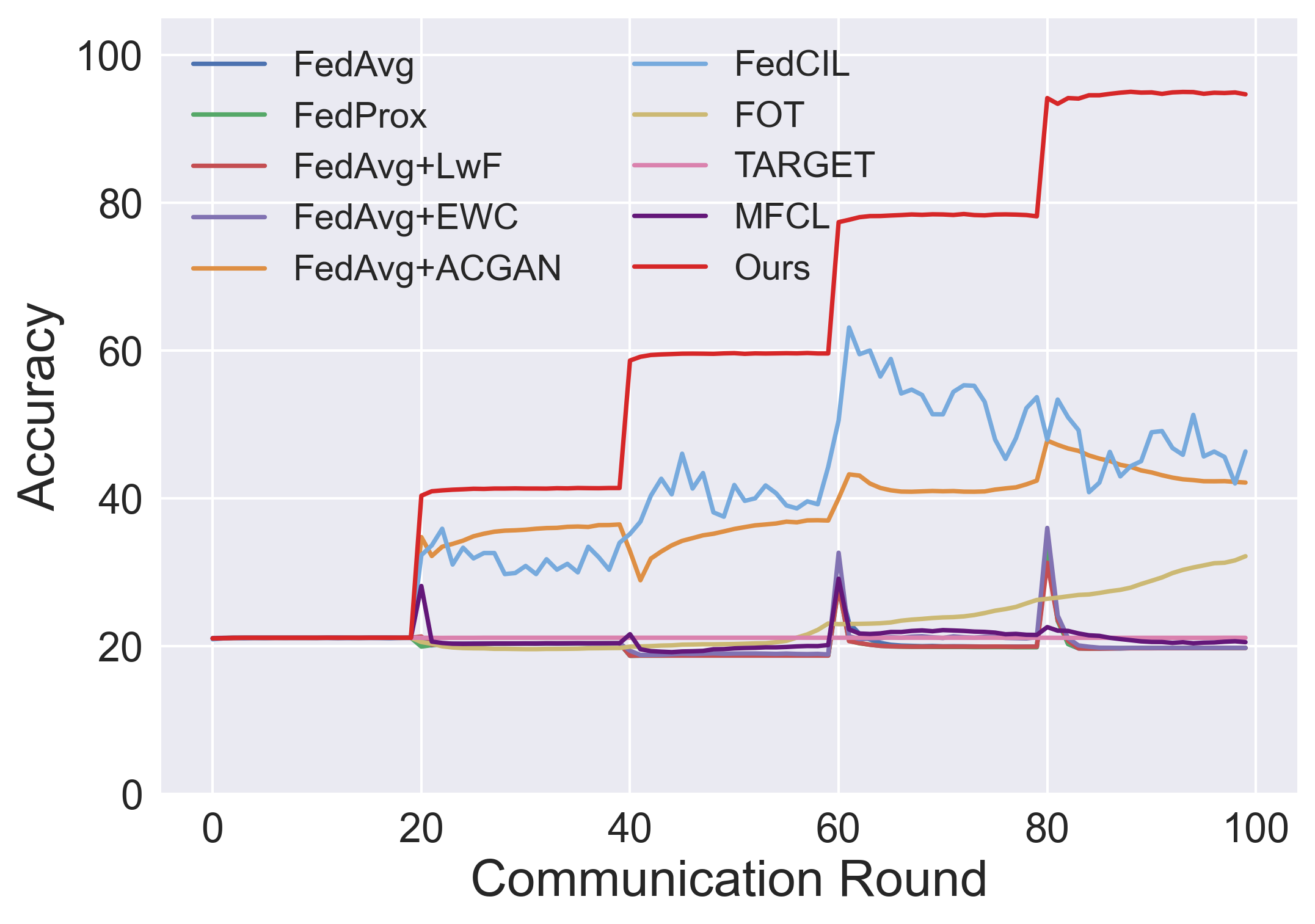}} &
\raisebox{-0.5\height}{\includegraphics[width=0.3\linewidth]{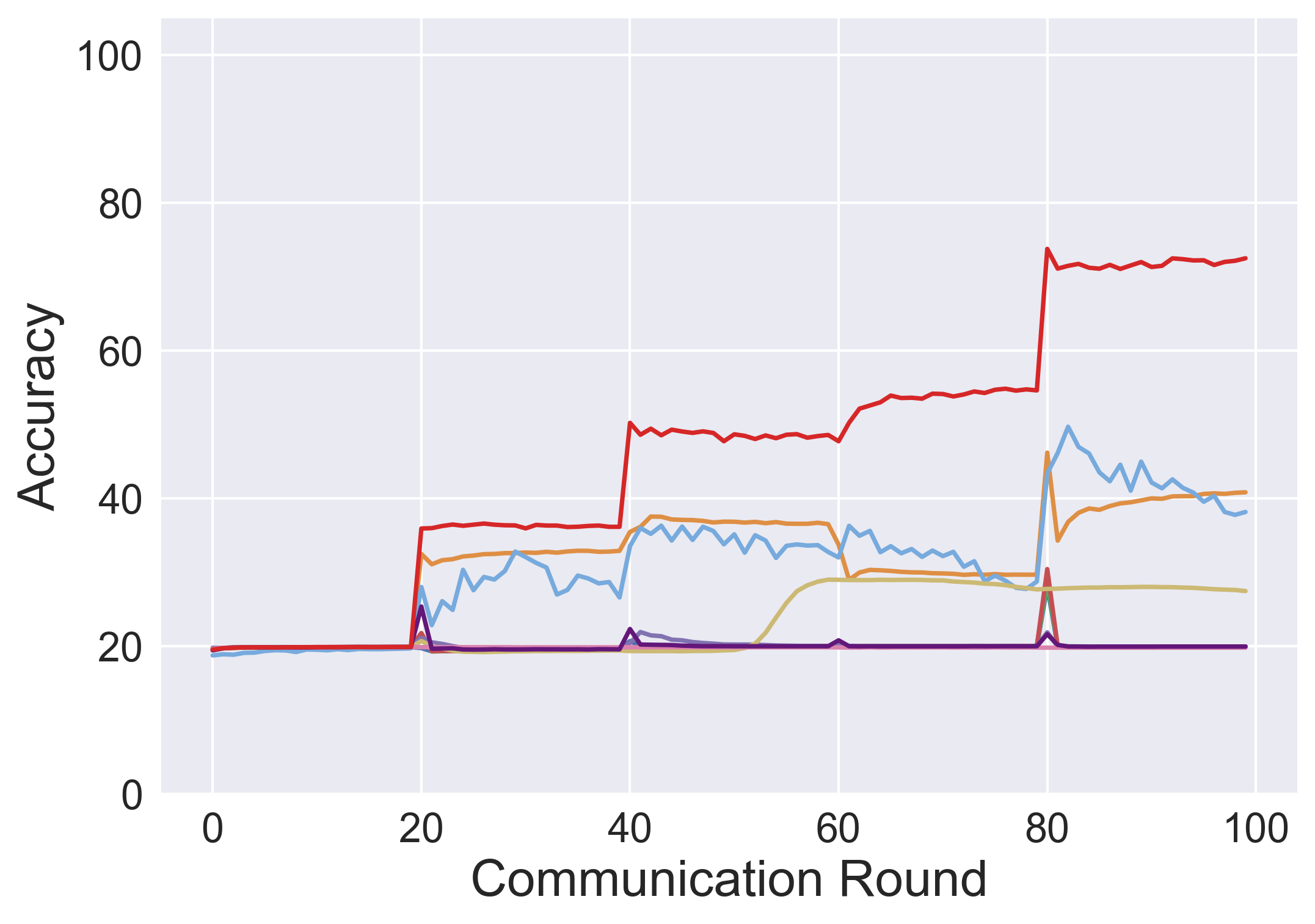}} &
\raisebox{-0.5\height}{\includegraphics[width=0.3\linewidth]{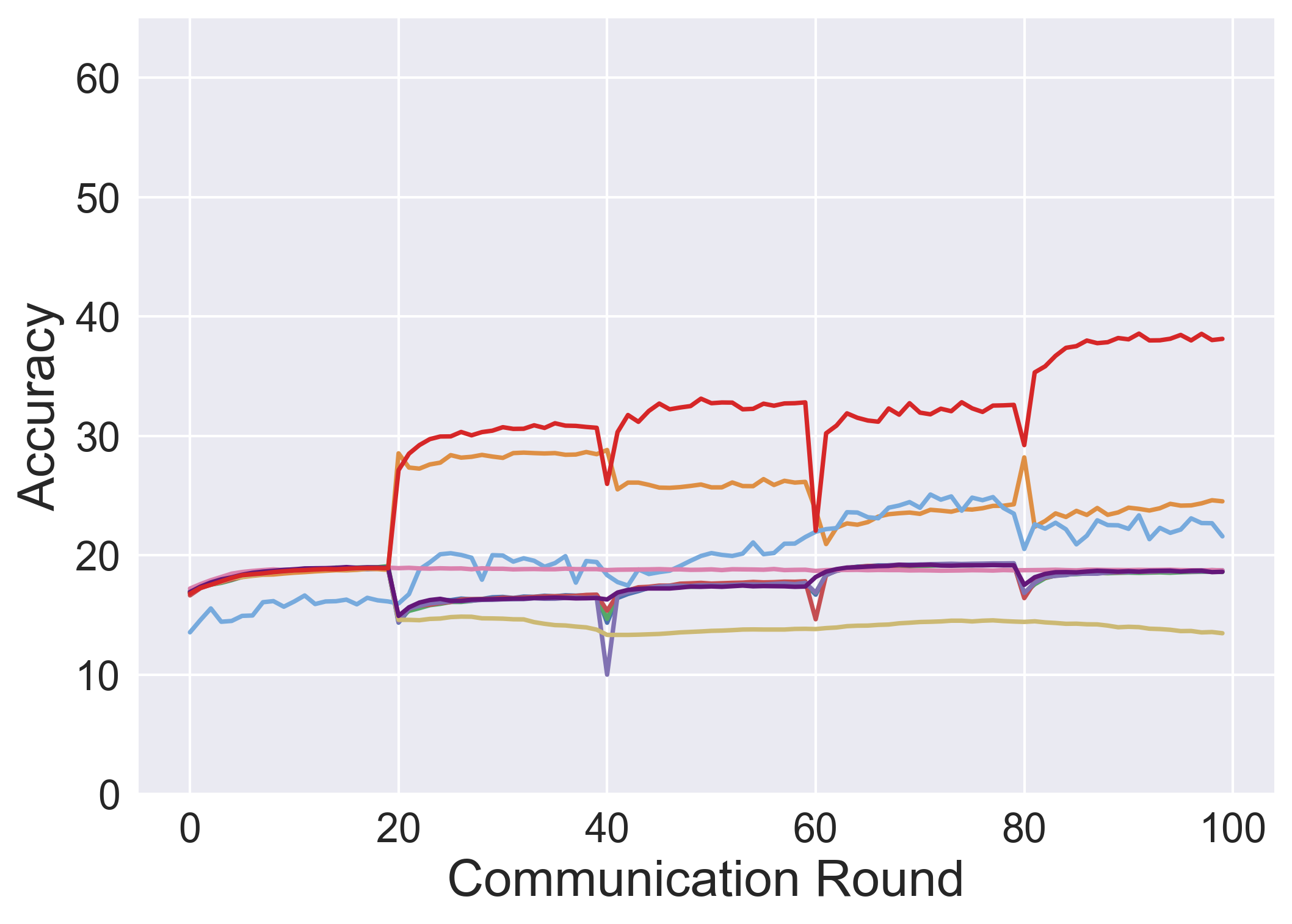}} \\
\rotatebox[origin=c]{90}{\textbf{Class Incremental Non-IID}} &
\raisebox{-0.5\height}{\includegraphics[width=0.3\linewidth]{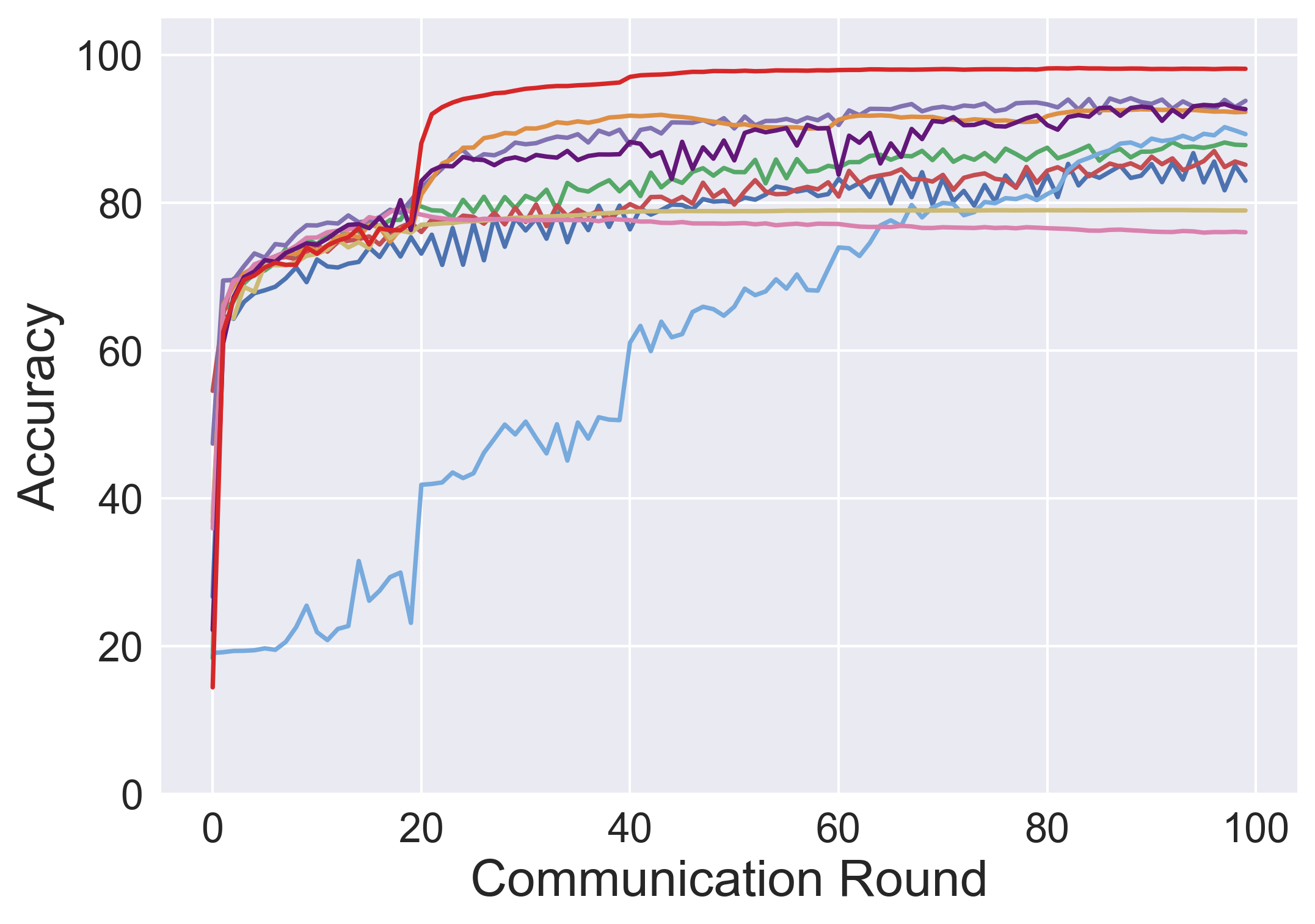}} &
\raisebox{-0.5\height}{\includegraphics[width=0.3\linewidth]{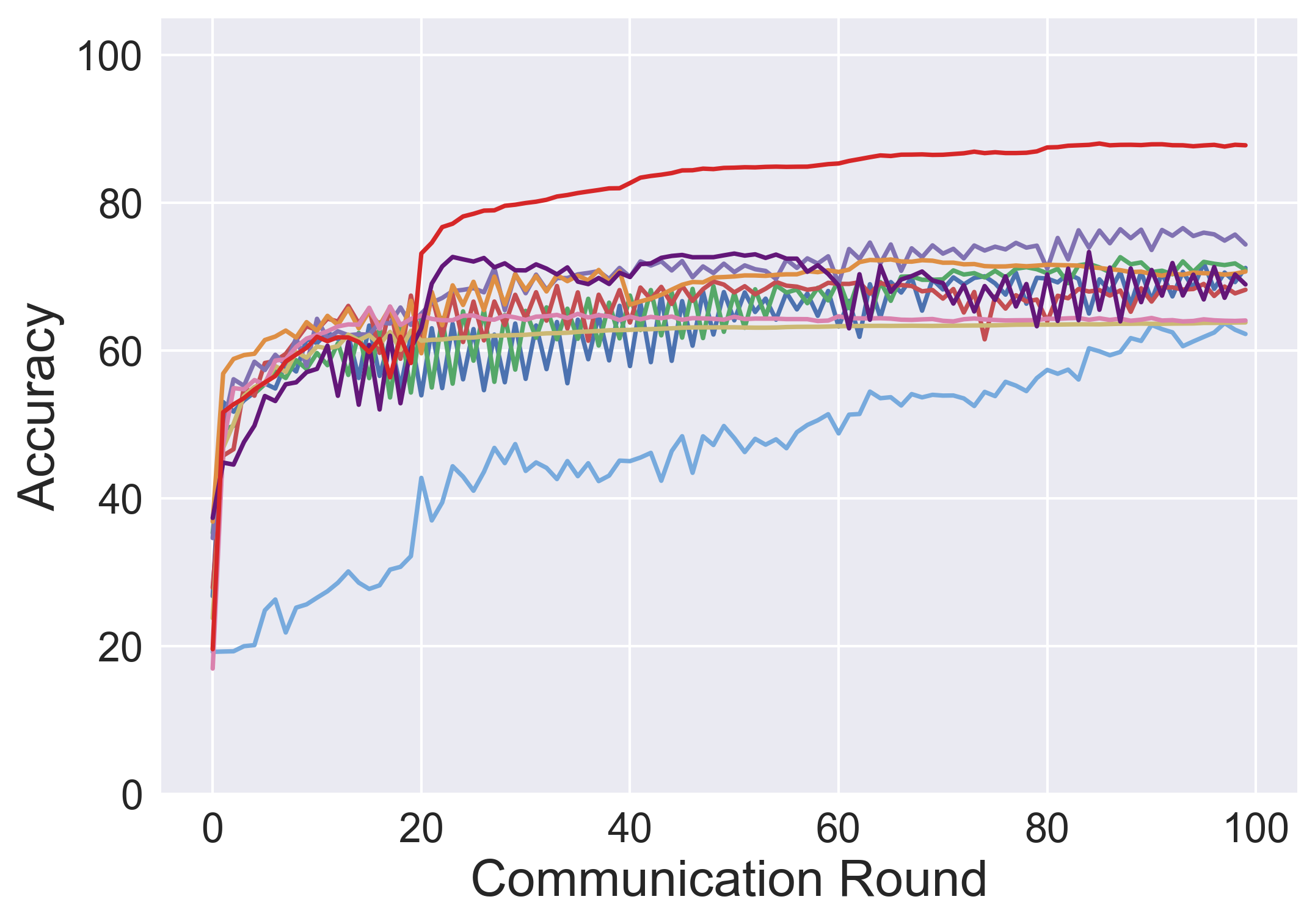}} &
\raisebox{-0.5\height}{\includegraphics[width=0.3\linewidth]{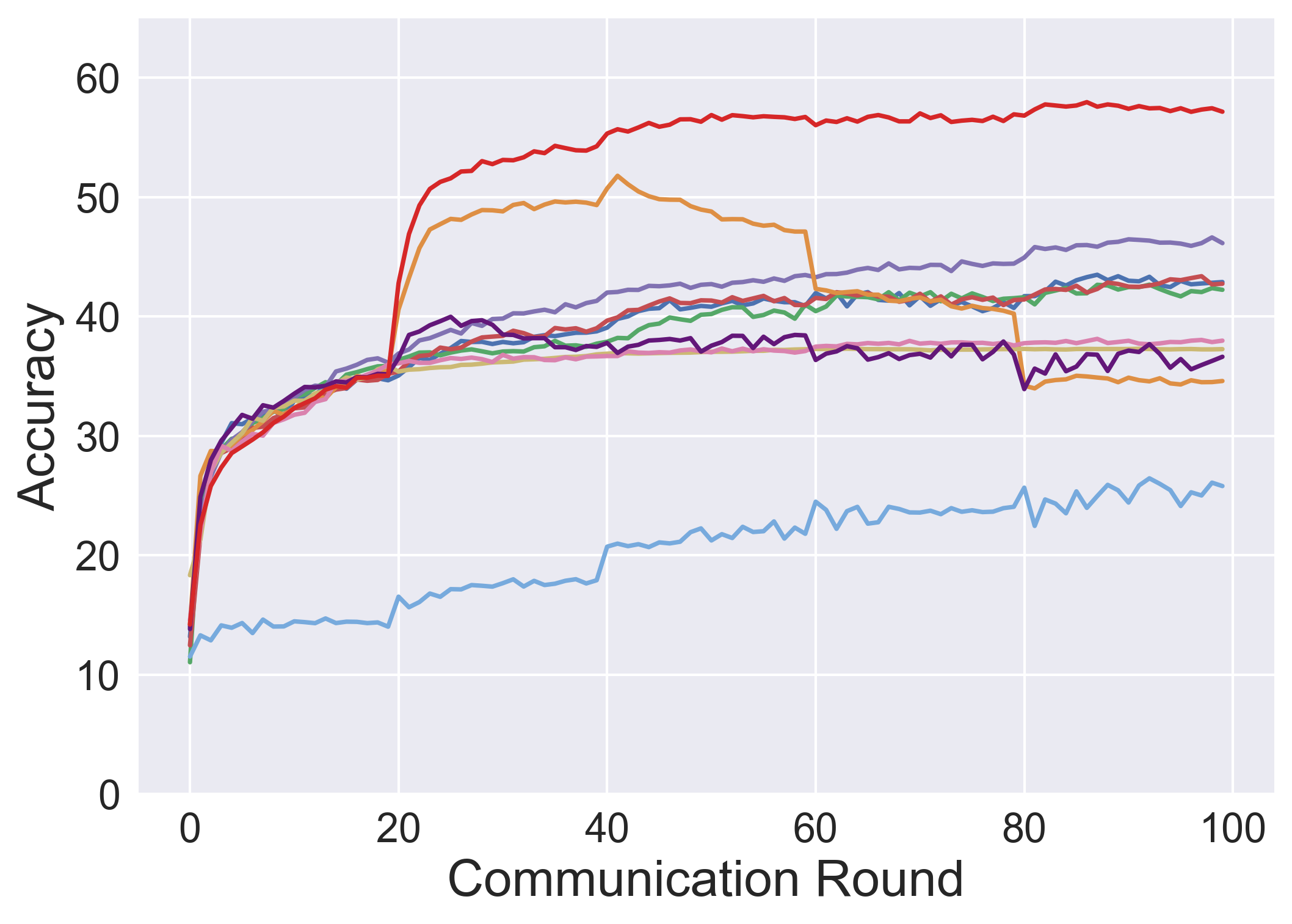}}
\end{tabular}
\caption{\textbf{Main Result} - Comparison of Model Convergence with Baselines. Refer to Figure \ref{Fig. Domain Incremental} for the Domain Incremental scenario.}
\label{Fig. Experimental results} 
\end{figure}

Table \ref{Table Main Result} shows that, compared to the baselines, our \NAME framework demonstrates significant advantages across all three datasets in both the Class Incremental IID and Class Incremental Non-IID scenarios. Specifically, the proposed approach achieves improvements of $32.61 \pm 15.91\%$ in the Class Incremental IID scenario (compared to the best baseline, FedAvg+ACGAN), $15.16 \pm 6.97\%$ in the Class Incremental Non-IID scenario (compared to the best baseline, FedAvg+EWC), and $7.45\%$ in the Domain Incremental scenario (compared to the best baseline, FedAvg+ACGAN). This shows that our approach is better at overcoming catastrophic forgetting than the baselines. Unlike FedAvg+ACGAN, which also relies on a generative model for replay, our use of a diffusion model generates higher-quality synthetic images, thereby avoiding noise contamination in the training dataset. Furthermore, the SOTA baselines suffer from their inherent limitations: FedCIL experiences unstable convergence due to multiple loss functions, while FOT demonstrates slow convergence, rendering it impractical. Meanwhile, both MFCL and TARGET produce low-quality synthetic images as the generative model is trained solely on the basis of knowledge distillation.

\newcommand{\com}[1]{\tiny \ \textcolor{Red}{$\downarrow$}#1}

\begin{table}[h]
\scriptsize
\caption{\textbf{Main Result} - Comparison of Final Accuracy with Baselines.} 
\label{Table Main Result}
\begin{tabular}{@{}lccccccc@{}}
\toprule
\textbf{Scenario} & \multicolumn{3}{c}{\textbf{Class Incremental IID}} & \multicolumn{3}{c}{\textbf{Class Incremental Non-IID}} & \textbf{Domain Incremental} \\
\cmidrule(lr){2-4} \cmidrule(lr){5-7} \cmidrule(lr){8-8}
\textbf{Dataset} & MNIST & FMNIST & CIFAR-10 & MNIST & FMNIST & CIFAR-10 & PACS \\

\midrule

FedAvg \cite{mcmahan2017communication} & 19.77\com{74.92} & 19.96\com{52.54} & 18.74\com{19.39} & 82.99\com{15.15} & 71.19\com{16.60} & 42.89\com{14.27} & 29.96\com{18.06} \\
FedProx \cite{li2020federated} & 19.78\com{74.91} & 19.96\com{52.54} & 18.60\com{19.53} & 87.81\com{10.33} & 70.93\com{16.86} & 42.24\com{14.92} & 33.82\com{14.21} \\
FedAvg+LwF \cite{li2017learning} & 19.77\com{74.92} & 19.96\com{52.54} & 18.67\com{19.46} & 85.16\com{12.98} & 68.24\com{19.55} & 42.75\com{14.41} & 34.57\com{13.46} \\
FedAvg+EWC \cite{kirkpatrick2017overcoming} & 19.78\com{74.91} & 19.95\com{52.55} & 18.64\com{19.49} & 93.81\com{4.33} & 74.37\com{13.42} & 46.15\com{11.01} & 38.02\com{10.01} \\
FedAvg+ACGAN \cite{odena2017conditional} & 42.15\com{52.54} & 40.83\com{31.67} & 24.52\com{13.61} & 92.31\com{5.83} & 70.70\com{17.09} & 34.60\com{22.56} & 40.57\com{7.45} \\
FedCIL \cite{qi2023better} & 46.36\com{48.33} & 38.17\com{34.33} & 21.59\com{16.54} & 89.30\com{8.84} & 62.25\com{25.54} & 25.80\com{31.36} & 27.79\com{20.24} \\
FOT \cite{bakman2023federated} & 32.18\com{62.51} & 27.47\com{45.03} & 13.47\com{24.66} & 78.97\com{19.17} & 63.82\com{23.97} & 37.27\com{19.89} & 39.42\com{8.60} \\
MFCL \cite{babakniya2024data} & 21.14\com{73.55} & 19.81\com{52.69} & 18.74\com{19.39} & 76.02\com{22.12} & 64.08\com{23.71} & 37.98\com{19.18} & 36.16\com{11.86} \\
TARGET \cite{zhang2023target} & 20.55\com{74.14} & 19.96\com{52.54} & 18.62\com{19.51} & 92.68\com{5.46} & 68.96\com{18.83} & 36.63\com{20.53} & 30.72\com{17.30} \\
\midrule
\NAME (Ours) & \textbf{94.69} & \textbf{72.50} & \textbf{38.13} & \textbf{98.14} & \textbf{87.79} & \textbf{57.16} & \textbf{48.02} \\
\bottomrule
\end{tabular}
\\[2pt] 
\com{ } \scriptsize indicates the accuracy decrease of the baselines compared to our \NAME framework. \hfill FMNIST refers to Fashion-MNIST.

\end{table}

\subsection{Analysis of Scenarios}

In the previous subsection, we compared the accuracy of \NAME with the baselines on the global test set. In this subsection, we discuss the performance within the scenarios. Specifically, we demonstrate the results using the standard CL evaluation method in Appendix \ref{appendix:Analysis}, considering only the classes or domains encountered so far, following the settings in CL. The Class Incremental Non-IID scenario is unsuitable for CL evaluation because it includes all classes at any given time.

\textbf{Class Incremental IID} is the most challenging CFL scenario because the server cannot aggregate global knowledge, as all clients only have the same subset of classes at any given time. Due to this characteristic, catastrophic forgetting is particularly severe. Figure \ref{Fig. Experimental results} shows that most traditional baselines completely fail to retain previous knowledge because all knowledge appears only in a single session. In this scenario, replay-based approaches are the best. This is because generative models can replay very old knowledge, whereas other approaches struggle to retain even the knowledge from the previous session. Among these, the diffusion model's high-quality synthetic images avoid error propagation, creating a more positive and effective feedback loop than other generative models like ACGAN. CL analysis of the Class Incremental IID scenario can be found in Appendix \ref{appendix:Analysis Class Incremental IID} and Tables \ref{Table Class Incremental IID 1} and \ref{Table Class Incremental IID 2}.

\textbf{Class Incremental Non-IID} is the simplest (i.e., the least affected by catastrophic forgetting) CFL scenario because the server can aggregate a generalized global model based on the clients with diverse class distributions. So that the clients can continually learn global knowledge of other classes from other clients. As shown in Figure \ref{Fig. Experimental results}, the Class Incremental Non-IID scenario mirrors the convergence process of traditional FedAvg in a Non-IID setting. Even with changes in client data, FedAvg can eventually converge based on its proven methodology. However, our \NAME framework accelerates this process, providing more stability and faster convergence than the baselines. Benefiting from the diffusion model, clients can learn features of multiple classes simultaneously, transforming the Non-IID problem into an IID problem: as clients iterate through communication rounds, they gradually acquire information about global classes through synthetic datasets.

\textbf{Domain Incremental} is a moderately challenging task. Similar to Class Incremental IID, it requires retaining knowledge of all domains without forgetting, as clients will not revisit them. However, we ensure that all clients have access to all classes within each domain, which makes it easier for clients to learn global class knowledge. The diffusion model excels at generating images across different domains and classes, outperforming all baselines. Notably, for Domain Incremental, we consider both domain and class as conditions for synthetic data generation, which is crucial due to the significant differences between domains in PACS that need special consideration. CL analysis of the Domain Incremental scenario can be found in Appendix \ref{appendix:Analysis Domain Incremental} and Figure \ref{Fig. Domain Incremental}, and Table \ref{Table Domain Incremental}.
    \section{Related Work}

\textbf{Federated Learning (FL)} has gained significant attention for its ability to train models across distributed data sources without centralizing data. A plethora of FL frameworks are designed for distributed learning, including asynchronous FL, decentralized FL, and hierarchical FL, among others, as well as numerous studies aimed at addressing the challenges of non-IID data and heterogeneity within FL \cite{pfeiffer2023federated,yuan2024decentralized,zhou2024every,zhou2022federated}. Some research has considered FL in dynamic contexts, such as varying communication conditions, client mobility status, client availability, and resource constraints \cite{liu2023dynamite,zhang2023joint,yuan2024communication,qiao2024br}. However, current research on the issue of dynamically varying client datasets (i.e., CFL) is not comprehensive, and existing solutions are limited to single CFL scenarios \cite{volpi2021continual,dong2022federated,park2024stablefdg}. In this work, we enumerate three CFL scenarios and experimentally demonstrate that \NAME is a universal solution.

\textbf{Continual Learning (CL)} approaches mitigate catastrophic forgetting \cite{nguyen2019toward} issues while sequentially learning new tasks. Among these, some methods introduce a replay memory where the experienced data can be stored \cite{rolnick2019experience,chaudhry2019tiny}. Others include structure-based methods \cite{yoon2017lifelong} and regularization-based approaches \cite{aljundi2018memory,kirkpatrick2017overcoming} to reduce the forgetting. Recently, leveraging generative replay \cite{shin2017continual,wu2018memory,qi2022better} offers a promising avenue for preserving past knowledge while adapting to new tasks, thus addressing the evolving nature of learning scenarios without storing the past data.
However, some traditional CL algorithms, such as Learning without Forgetting (LwF) \cite{li2017learning}, may not be applicable if there is no overlap between the previous and current data distributions. LwF typically assumes that new tasks share commonalities with previous tasks, enabling the model to maintain old knowledge while learning new information. This assumption breaks down when previous tasks' classes (e.g., $\{0,1\}$) are completely different from those of the new tasks (e.g., $\{2,3\}$). Therefore, our proposed diffusion model as replay has been demonstrated to be a universal solution for all different CFL scenarios.

\textbf{Diffusion Models} have showcased superior performance in generating detailed and diverse instances \cite{yang2023diffusion} and succeeded in many areas, including computer vision \cite{mei2023exploiting,weng2024video} and reinforcement learning \cite{fang2024learning,chen2024deep}. There exist three main formulations of diffusion models: Score-based Generative Models (SGM)~\cite{song2019generative,song2020improved}, Denoised Diffusion Probabilistic Models (DDPM)~\cite{sohl2015deep,ho2020denoising,nichol2021improved}, and Stochastic Differential Equations (Score SDE)~\cite{song2020score,song2021maximum}. On this basis, \cite{dhariwal2021diffusion} shows the superiority of the diffusion model compared to other generative models via training a classifier on noisy images and using gradients to guide the diffusion sampling process to the conditioning information, such as labels, by altering the noise prediction. Besides, \cite{ho2021classifierfree} also shows the possibility of running conditional diffusion without an independent classifier. In this work, we exploit a conditional diffusion model in our CFL framework for historical data recovery.

\section{Conclusion}

In this paper, we introduce \NAME, a novel Continual Federated Learning (CFL) framework that incorporates diffusion models for synthetic historical data generation. The synthetic data generated by \NAME helps in retaining memory of previously encountered input data distributions and mitigates the impact of data distribution shifts during learning, thus avoiding latent catastrophic forgetting issues. Theoretical analyses are provided to support the convergence of our proposed framework. Furthermore, experimental results on multiple datasets showcase the effectiveness of \NAME in addressing FL tasks and mitigating the negative impact of input distribution shifts. Currently, as one limitation, we have not explored utilizing multimodality data, such as text prompts, for enhanced data synthesis in CFL tasks, which can be considered as future work.
    
    \newpage
    \bibliography{reference.bib}
    \bibliographystyle{unsrt}

    \newpage
\appendix

\section{Proof of Theorem \ref{theorem:data_bound}}
\label{proof:data_bound}

\begin{proof}
	Given the KL divergence, we have:
	\begin{equation}
		\begin{aligned}
			&D_\mathrm{KL}\left( F^*_t(\theta^k_t;\xi^k_t) \parallel F^*_{t+1}(\theta^k_{t+1};\xi^k_{t+1}) \right) \\
			&= \sum F^*_t(\theta^k_t;\xi^k_t) \log \left(\frac{F^*_t(\theta^k_t;\xi^k_t)}{F^*_{t+1}(\theta^k_{t+1};\xi^k_{t+1})}\right) \\
			&= \sum F^*_t(\theta^k_t;\xi^k_t) \frac{1}{2} \log \left[ \frac{F^*_t(\theta^k_t;\xi^k_t)}{(1+\delta)^{-1}\left(F^*_{t+1}(\theta^k_{t+1};\tilde{\xi}^k_{t+1}) + F^*_{t+1}(\theta^k_{t+1};\tilde{\tilde{\xi}}^k_{t+1})\right)} \right]^2 \\
			&\leq \sum F^*_t(\theta^k_t;\xi^k_t) \frac{1}{2} \left(\log  \frac{F^*_t(\theta^k_t;\xi^k_t)}{(1+\delta)^{-1} F^*_{t+1}(\theta^k_{t+1};\tilde{\xi}^k_{t+1})} + \log \frac{F^*_t(\theta^k_t;\xi^k_t)}{(1+\delta)^{-1} F^*_{t+1}(\theta^k_{t+1};\tilde{\tilde{\xi}}^k_{t+1})} \right) \\
			&= \frac{1}{2}\left[ D_\mathrm{KL}\left(F^*_t(\theta^k_t;\xi^k_t) \parallel F^*_{t+1}(\theta^k_{t+1};\tilde{\xi}^k_{t+1})\right) + D_\mathrm{KL}\left( F^*_t(\theta^k_t;\xi^k_t) \parallel F^*_{t+1}(\theta^k_{t+1};\tilde{\tilde{\xi}}^k_{t+1}) \right) \right] \\
			&\stackrel{(a)}{=} \frac{1}{2}\left[ D_\mathrm{KL}\left( F^*_t(\theta^k_t;\xi^k_t) \parallel F^*_{t+1}(\theta^k_{t+1};\tilde{\tilde{\xi}}^k_{t+1}) \right) + \Delta_t \right],
		\end{aligned}
	\end{equation}
	where $(a)$ adopts Assumption~\ref{assumption:shift_bound}. 
 
    This concludes the proof.
\end{proof}

\section{Proof of Theorem \ref{theorem:cfl_bound}}
\label{proof:cfl_bound}

\begin{proof}
    Considering the generated data at step $t$ will be the input data of step $t + 1$, we have the iterative measurement of the bound regarding this data distribution shift. According to Theorem \ref{theorem:data_bound} and \eqref{eq:distribution}, we have:
    \begin{equation}
        \begin{aligned}
            &D_\mathrm{KL}\left( F^*_{t}(\theta^k_{t};\xi^k_{t}) \parallel F^*_{t+1}(\theta^k_{t+1};\xi^k_{t+1}) \right) \\
            &\leq \frac{1}{2}\left[ D_\mathrm{KL}\left( F^*_{t}(\theta^k_{t};\xi^k_{t}) \parallel F^*_{t+1}(\theta^k_{t+1};\tilde{\tilde{\xi}}^k_{t+1}) \right) + \Delta_t \right] \\
            &= \frac{1}{2}\left[ D_\mathrm{KL}\left( F^*_{t}(\theta^k_{t};\xi^k_{t}) \parallel F^*_{t+1}(\theta^k_{t+1};\tilde{\tilde{\xi}}^k_{t+1}) \right) + \frac{1}{2}\left[ D_\mathrm{KL}\left( F^*_{t+1}(\theta^k_{t+1};\xi^k_{t+1}) \parallel F^*_{t+2}(\theta^k_{t+2};\tilde{\tilde{\xi}}^k_{t+2}) \right) + \Delta_{t+1} \right] \right] \\
            &\cdots \\
            &\stackrel{(a)}{\leq} \left(1 - \frac{1}{2^T}\right) \left(\epsilon^2_{\rm score} + \kappa^2 dM + \kappa d\bar{T} + d\exp(-2\bar{T}) \right) + \frac{1}{2^T}\Delta_T,
        \end{aligned}
        \label{eq:iterative}
    \end{equation}
    where we further simplify the results at the last inequality (a) using the summation of the geometric series and Lemma \ref{lemma:diffusion_bound}.

    We use the FedAvg model as the backbone for federated learning and the diffusion model to generate synthetic data for continual learning. The convergence of the system takes into account the inevitable shift in distribution, which can be measured using KL divergence. This distribution shift occurs at each iteration from $t=1$ to $T$, nested as shown in \eqref{eq:iterative}. Initially, the convergence is determined by:
    \begin{equation}
        \begin{aligned}
            \E \left[ F(\theta_T)\right] - F^*_{T} \leq& \E \left[ F(\theta_T)\right] + D_\mathrm{KL}(F^*_{1} \parallel F^*_{2}) - F^*_T.
        \end{aligned}
    \end{equation}

    \Eqref{eq:iterative} has provided us with an upper bound about the introduced data distribution shift ranging from $t=1,\dots,T$. Considering this in the designed CFL system, the final convergence bound can be derived as:
    \begin{equation}
        \begin{aligned}
            \E \left[ F(\theta_T)\right] - F^*_{T} &\leq A + D_\mathrm{KL}(F^*_{1} \parallel F^*_{2}) \\
            &\leq A + \left(1 - \frac{1}{2^T}\right) \left(\epsilon^2_{\rm score} + \kappa^2 dM + \kappa d\bar{T} + d\exp(-2\bar{T}) \right) + \frac{1}{2^T}\Delta_T.
        \end{aligned}
    \end{equation}
    where $A \triangleq \frac{\kappa}{\gamma +T-1} \left( \frac{2B}{\mu} + \frac{\mu \gamma}{2} \E \Vert\theta_1 - \theta^*\Vert^2 \right)$ based on Lemma \ref{lemma:fedavg_bound}. This concludes the proof.
\end{proof}

\section{Complete algorithm}
\label{appendix:alg}

We present our complete algorithm combining FL with the diffusion model below as the complete extension of the Algorithm \ref{alg:cfl}. As noted in Algorithm \ref{alg:cfl_comp}, we illustrate the alignment of the diffusion model with local clients. In practice, depending on the task's complexity, we may transition from an unguided diffusion model, such as a common DDPM diffusion model, to a conditioned diffusion model. While the former is sufficient for most simple generation workloads, the latter performs better for tasks requiring complicated synthetic unstructured data generation, such as images with delicate contents. To utilize the conditioned diffusion model, both the data and its label are passed as inputs for model training. This often necessitates thorough model training to achieve precise data generation, demanding significant computational resources. To address this challenge, we can leverage a pre-trained model trained on a large general dataset and fine-tune our diffusion model accordingly based on our purposes.

\begin{algorithm}[th]
    \caption{Proposed \NAME Framework Complete Procedure}
    \label{alg:cfl_comp}
    
        \textbf{Input:} Communication rounds ($T$), client datasets ($\gD^k_t$), \textcolor{Tan}{target model, loss function, and learning rate ($\theta$, $F$, $\eta_\theta$)}, \textcolor{RoyalBlue}{diffusion model, diffusion step, loss function, and learning rate ($\omega$, $N$, $\gL$, $\eta_\omega$)}
    
        \textbf{Output:} Generalized global model ($\theta_T$)
        \vspace{5pt}
    
    \begin{algorithmic}[1]
        \Statex \textbf{Local update} of the $k$-th client:
            \State \textbf{initialize} $\omega_0$
        \For{each round $t=1:T$}
        \If{$t = 0$}
            \State Data includes current real data only: $\gX^k_0=\gD^k_0$
            \ElsIf{$t > 1$}
        \State \textcolor{RoyalBlue}{$\rhd$ Generate synthetic data}
        \State $\gG^k_{t-1,n} \sim \gN(0,\mI)$
        \For{$n=N:1$}
        \State $\vz\sim\gN(0,\mI)$ if $t>1$ else $\vz={\bf 0}$
        \State $\gG^k_{t-1,n-1} = \frac{1}{\sqrt{\alpha_n}}\left(\gG^k_{t-1, n} - \frac{1-\alpha_n}{\sqrt{1-\bar{\alpha}_n}}\epsilon_\omega(\gG^k_{t-1, n},\vy,n)\right) + \sqrt{\beta_n} \vz$, given labels $\vy$
        \EndFor
        \State Obtain synthetic data $\gG^k_{t-1} \equiv \gG^k_{t-1, 0}$

        \State Combine real and synthetic data with a scale factor $\delta$: $\gX^k_t=\gD^k_t \cup \delta\cdot \{\gG^k_{t-1}, \vy\}$
        \EndIf

        \Statex \hfill  -- $\cdot$ -- $\cdot$ -- $\cdot$ -- $\cdot$ -- $\uparrow$ Generating synthetic data $\uparrow$ -- $\cdot$ -- $\cdot$ -- $\cdot$ -- $\cdot$ -- $\downarrow$ Training models $\downarrow$ -- $\cdot$ -- $\cdot$ -- $\cdot$ -- $\cdot$ -- \hfill 

        \Repeat {} $E_\theta$ epochs
        \State $\theta^k_{t} \leftarrow \theta^k_{t-1} - \eta_\theta\nabla F_k(\theta^k_{t-1};\gX^k_t)$ \hfill \textcolor{Tan}{$\rhd$ Train target model}
        \Until{$\theta$ converged}
        \Repeat {} $E_\omega$ epochs
        \State $\omega^k_{t} \leftarrow \omega^k_{t-1} - \eta_\omega\nabla \gL_k(\omega^k_{t-1};\gX^k_t, n)$ \hfill \textcolor{RoyalBlue}{$\rhd$ Train diffusion model}
        \Until{$\omega$ converged}
        
        \EndFor
        \State \textbf{return} $\theta^k_{t}$ to server
        \Statex

\end{algorithmic}
\begin{algorithmic}[1]

		\Statex \textbf{Global update} of the server
		\State \textbf{initialize} $\theta_0$
		\For{each round $t=1:T$}
		\For{each client $k=1:K$ \textbf{in parallel}}
		\State $\theta^k_{t} \leftarrow$ $k$-th client's \textit{local update}
		\EndFor
		\State $\theta_t \leftarrow \sum_{k=1}^{K}p_k\theta^k_t$ \hfill \textcolor{Tan}{$\rhd$ Aggregate target models}
		\EndFor

	\end{algorithmic}
\end{algorithm}

\newpage

\section{Datasets}
\label{appendix:Datasets}

We use the following four datasets for our experiments, each applied to different CFL scenarios, as shown in Table \ref{Table Dataset Property}.

\begin{table}[h]
\small
\caption{\textbf{Properties of Datasets and Three CFL Scenarios.}}
\label{Table Dataset Property}
\centering
\begin{tabular}{@{}l|c|c|c@{}}
\toprule
\textbf{Scenario} & \textbf{Class Incremental IID} & \textbf{Class Incremental Non-IID} & \textbf{Domain Incremental} \\
\midrule
\multirow{3}{*}{\textbf{Dataset}} & MNIST & MNIST & PACS \\
& Fashion-MNIST & Fashion-MNIST &  \\
& CIFAR-10 & CIFAR-10 &  \\
\midrule
\multirow{2}{*}{\textbf{Attribute}} & Client 0:$ \{0, 1\} \rightarrow \{2, 3\}$ & Client 0:$ \{0, 1\} \rightarrow \{2, 3\}$ & Client 0: Domain 0 $\rightarrow$ 1\\
& Client 1:$ \{0, 1\} \rightarrow \{2, 3\}$ & Client 1:$ \{2, 3\} \rightarrow \{4, 5\}$ & Client 1: Domain 0 $\rightarrow$ 1\\
\midrule
\textbf{Total Rounds $T$} & 100 & 100 & 80 \\
\midrule
\textbf{Session/Domain} & 5 & 5 & 4 \\
\midrule
\textbf{Number of Classes} & All Clients: 2 & All Clients: 10 & All Clients: 7 \\
\textbf{at Any Time} & Each Client: 2 & Each Client: 2 & Each Client: 7 \\

\bottomrule
\end{tabular}
\end{table}

\begin{itemize}[leftmargin=*]
    \item \textbf{Class-Incremental IID} and \textbf{Class-Incremental Non-IID.}
    \begin{itemize}
        \item \textbf{MNIST} \cite{lecun1998gradient} is a large dataset for handwritten digit recognition, containing 70,000 grayscale images of size $28\times28$ pixels, representing digits from 0 to 9. To accommodate the input dimensions of the UNet in our diffusion model, we resize the images to $32\times32$ pixels. We split the MNIST dataset into 20 clients, with each client having 5 sessions and each session containing 2 classes. Therefore, each session in every client has 300 samples for any given class.
        \item \textbf{Fashion-MNIST} \cite{xiao2017fashion} is a large dataset for clothing recognition, similar in size and format to MNIST, with 70,000 grayscale images of size $28\times28$ pixels. It includes categories such as T-shirt, Trouser, Pullover, Dress, Coat, Sandal, Shirt, Sneaker, Bag, and Ankle boot. Compared to MNIST, the images in Fashion-MNIST have more complex shapes and textures, increasing the difficulty of recognition. We adopt the same loading, partitioning, and preprocessing procedures as with MNIST.
        \item \textbf{CIFAR-10} \cite{krizhevsky2009learning} is a large dataset for object recognition, consisting of 60,000 color images of size $32\times32$ pixels, each with three RGB color channels. The dataset includes objects such as Airplane, Automobile, Bird, Cat, Deer, Dog, Frog, Horse, Ship, and Truck. Compared to MNIST and Fashion-MNIST, CIFAR-10 poses a greater challenge due to the complexity and diversity of backgrounds and objects. We split the CIFAR-10 dataset into 10 clients, with each client having 5 sessions and each session containing 2 classes. Therefore, each session in every client has 500 samples for any given class.
    \end{itemize}
    \item \textbf{Domain-Incremental.}
    \begin{itemize}
        \item \textbf{PACS} \cite{li2017deeper} is a widely used dataset for domain adaptation and generalization studies, featuring 4 significantly different visual domains: Photo, Art Painting, Cartoon, and Sketch. It contains a total of 9,991 color images of size 227x227 pixels, with each domain comprising: Photo (1,670 images), Art Painting (2,048 images), Cartoon (2,344 images), and Sketch (3,929 images). Each domain includes seven categories: Dog, Elephant, Giraffe, Guitar, Horse, House, Person. We use the PACS dataset for domain-incremental learning, ensuring that each client retains the same categories across different domains (the same as the IID setting in FedAvg). The entire PACS dataset is split into 80\% training and 20\% testing sets, with the training set further divided among clients according to the domain-incremental setting. We split the PACS dataset into 10 clients, with each client having 4 domains, each containing all classes. Consequently, the number of samples per domain in each client is approximately 312, 184, 161, and 131, respectively. Note that the sample sizes for each domain in PACS are different; and these are approximate values because we randomly split the training and testing sets in an 8:2 ratio, and different seeds result in varying sample sizes (we ensure that the number of samples per class is consistent within each client). These sample sizes include all classes, but the class distribution differs across domains, as shown in Figure \ref{Fig. PACS_Distribution}.
    \end{itemize}
\end{itemize}

\begin{figure}[h]
    \centering
    \includegraphics[width=0.6\linewidth]{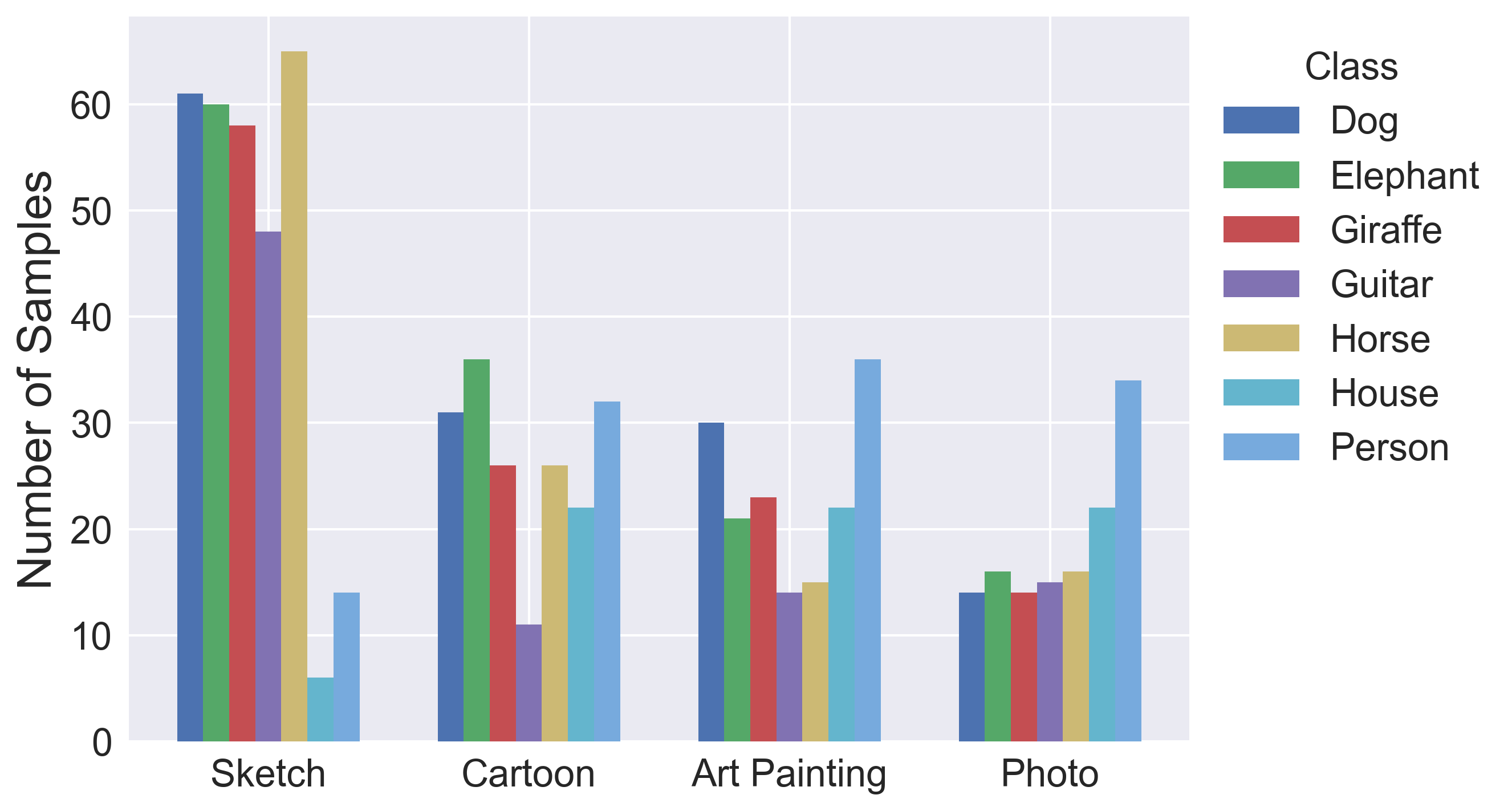}
    \caption{\textbf{Data Distribution of PACS Dataset.} The figure shows the number of samples owned by each one single client.}
    \label{Fig. PACS_Distribution}
\end{figure}

\section{Implementation Details}
\label{appendix:Implementation}

\textbf{Training Details.} For all experiments, both the target model and the diffusion model use the Adam optimizer with a learning rate of 1e-4 and a batch size of 32. For the target model, the training epochs are set to $E_\theta=5$. For the diffusion model, the training epochs are set according to the complexity of the images: $E_\omega=100$ for MNIST and Fashion-MNIST, and $E_\omega=1000$ for CIFAR-10 and PACS. All models are implemented in PyTorch and trained on an NVIDIA A100 GPU with 40 GB of memory. Using MNIST in the Class Incremental IID scenario as an example, the execution time for a single client is approximately 0.5 seconds for training the target model, 180 seconds for training the diffusion model, and 200 seconds for generating synthetic samples.

\textbf{Target Model - CNN Architecture.} We use a CNN as the target model for MNIST, Fashion-MNIST, CIFAR-10, and PACS. The architecture consists of two convolutional layers, the first with 32 filters and the second with 64 filters, both using a kernel size of 5 and padding of 2. Each convolutional layer is followed by a ReLU activation and a max pooling layer with a pool size of 2. The output from the convolutional layers is then flattened and passed through a fully connected layer with 512 units, followed by another ReLU activation. The final layer is a fully connected layer with 10 units (or 7 units for PACS), corresponding to the number of classes, and a log-softmax activation function.

\textbf{Diffusion Model - Conditional UNet Architecture.} We use a conditional UNet as the backbone for the diffusion model in our framework. The architecture starts with an initial convolutional layer that projects the input image into a higher-dimensional space with 64 channels. It then processes the data through a series of downsampling and upsampling blocks. The downsampling path consists of four blocks with channel sizes of 64, 128, 128, and 256, respectively. Each block contains residual blocks and, in the deepest layer, attention mechanisms to enhance feature representation. Each residual block includes group normalization and dropout, ensuring stable training and preventing overfitting. In the middle of the network, a middle block contains both residual and attention mechanisms, maintaining the channel size at 256. The upsampling path mirrors the downsampling path but in reverse order, reducing the channel dimensions while merging features from corresponding downsampling layers through skip connections. This upsampling process includes upsampling layers to expand the spatial dimensions back to the original size. Finally, the model normalizes and activates the features before passing them through a final convolutional layer, which reduces the output to the desired number of image channels. We incorporate time and condition information (class for MNIST, Fashion-MNIST, and CIFAR-10; class and domain for PACS) into the model, where the condition dimension is set to 32 for MNIST, Fashion-MNIST, and CIFAR-10, and 64 for PACS. The original implementations of DDPM and UNet are sourced from labml\_nn library \cite{labml}.

\newpage

\section{Baselines}
\label{appendix:Baselines}

We consider four types of baseline for comparison.

\begin{itemize}[leftmargin=*]
    \item \textbf{FL Algorithms.}
    \begin{itemize}
        \item \textbf{FedAvg} \cite{mcmahan2017communication} is the most representative and classic FL algorithm, where the server aggregates the received client models weighted by the number of client samples to obtain the global model.
        \item \textbf{FedProx} \cite{li2020federated} is popular FL algorithm. Built on FedAvg, it introduces a proximal term in the loss function during local training to reduce the divergence between the local and global models, addressing client heterogeneity. We did not adapt this proximal term loss to the CFL scenario, but only compute the difference between the global model at round $t-1$ and the local model at round $t$. The proximal term loss function is $\frac{\mu}{2} \| \theta^k_t - \theta_{t-1}\|^2$. We set the weight parameter $\mu=1$ for the proximal term as in the original paper.
    \end{itemize}

    \item \textbf{FL with Classical CL Methods.}
    \begin{itemize}
        \item \textbf{FedAvg+ \quad LwF} (Learning without Forgetting) \cite{li2017learning} is a classic CL algorithm that utilizes knowledge distillation, requiring the new model to mimic the output of the old model on the same input when training on new tasks. We adapt LwF to the CFL scenario, where, in session $s$, the student model is guided by the teacher model from the last communication round of session $s-1$. In which the last communication round of the previous session can be calculated as $t^\ast = \left\lfloor \frac{t}{T/S} \right\rfloor \times \frac{T}{S}$, where $T$ is the total number of rounds, and $S$ is the number of sessions. The teacher model $\theta_{t^\ast}$ guides all training in the next session. The LwF loss function is expressed as $\mathcal{L}_{\mathrm{LwF}} = \mathcal{L}_{\mathrm{task}} + \lambda_{\mathrm{LwF}} D_\mathrm{KL}\left(\theta_{t^\ast}(\mathbf{x}) \| \theta_t(\mathbf{x})\right)$, where $D_\mathrm{KL}$ denotes the Kullback-Leibler (KL) divergence, and $\lambda_{\mathrm{LwF}}$ is a weighting factor used to balance the influence of previous and current tasks. We set $\lambda_{\mathrm{LwF}}=1$, consistent with the setting in \cite{li2017learning}.
        \item \textbf{FedAvg+ \quad EWC} (Elastic Weight Consolidation) \cite{kirkpatrick2017overcoming} is another classic CL algorithm that applies additional constraints to protect model parameters relevant to old tasks from significant updates. We adapt EWC to the CFL scenario similarly to LwF, using the model from the last communication round of the previous session $\theta_{t^\ast}$ as the old task model to guide all training in the next session. The EWC loss function is $\mathcal{L}_{\mathrm{EWC}} = \mathcal{L}_{\mathrm{task}} + \frac{\lambda_{\mathrm{EWC}}}{2} \sum_{i} F^{t^\ast}_i (\theta^t_i - \theta^{t^\ast}_i)^2$, where $\lambda_{\mathrm{EWC}}$ is the EWC penalty, and $F_i$ are the values of the Fisher information matrix for the parameter $\theta^{t^\ast}_i$. We set $\lambda_{\mathrm{EWC}}=400$, consistent with the setting in \cite{kirkpatrick2017overcoming}.
    \end{itemize}
    \item \textbf{FL with Generative Model as Replay.}
    \begin{itemize}
        \item \textbf{FedAvg+ \quad ACGAN} (Auxiliary Classifier Generative Adversarial Networks) \cite{odena2017conditional} is a widely used conditional generative model. A traditional GAN consists of a generator that creates synthetic images and a discriminator that distinguishes between real and fake samples. ACGAN also extends this by requiring the discriminator to classify the samples, enabling conditional output. For the implementation of FedAvg+ACGAN, we replace the diffusion model in our \NAME framework with ACGAN while keeping the rest of the experimental settings identical, such as training epochs, learning rate, and the number of generated images.
    \end{itemize}
    \item \textbf{SOTA CFL Frameworks.}
    \begin{itemize}
        \item \textbf{FedCIL} \cite{qi2023better} uses ACGAN for generating synthetic data while employing model consolidation to aggregate ACGANs from different clients and consistency enforcement to ensure that local training aligns better with the global model, thereby reducing training bias. However, FedCIL has some significant limitations, including (i) the server requires resources to train a server ACGAN, (ii) clients need to upload their ACGANs, causing privacy leakage, and (iii) it is not applicable to different tasks, such as object detection and semantic segmentation, since FedCIL only considers the discriminator in FL. In contrast, decoupling the generative model from the target model, as done in our \NAME framework and the baseline \textbf{FedAvg+ACGAN}, offers better scalability.
        \item \textbf{FOT} \cite{bakman2023federated} performs global principal subspace extraction to identify features critical to previous tasks, which are then protected during subsequent training to prevent forgetting. Subsequently, through the orthogonal projection aggregation method, when training new tasks, the server orthogonally projects model updates from the clients onto the orthogonal complement of the old tasks' subspace. This ensures that updates occur only in directions unrelated to the old tasks, thereby minimizing the impact of learning new tasks on the performance of previous tasks. Although this approach does not require additional computation on the clients, it does necessitate additional computation on the server, such as orthogonal projection and subspace extraction. Moreover, FOT entails higher communication costs between clients and servers to transfer subspace information and orthogonal projections, raising concerns about potential privacy breaches.
        \item \textbf{MFCL} \cite{babakniya2024data} trains a generative model on the server through knowledge distillation using the aggregated global model and then downloads this generative model to all clients to generate synthetic data and perform local training. MFCL optimizes four loss functions for training the generative model: cross entropy, diversity, batch statistics, and image prior. Although this is a data-free generative model training strategy, the quality of the generated models is suboptimal. While clients incur no additional computational overhead, the server requires substantial computational resources to train the generative model and incurs significant communication overhead to transmit the model.
        \item \textbf{TARGET} \cite{zhang2023target} trains a generative model on the server by knowledge distillation using the aggregated global model and then downloads the synthetic dataset to all clients for local training. TARGET optimizes the generative model using cross-entropy, KL loss with the student model, and batch normalization. The success of local training is heavily dependent on the quality of the synthetic dataset, which is difficult to guarantee. Additionally, the server incurs substantial computational overhead for training the generative model and significant communication overhead for transmitting the synthetic dataset.

    \end{itemize}
\end{itemize}

\section{Continual Learning Analysis of Scenario}
\label{appendix:Analysis}

\subsection{Class Incremental IID}
\label{appendix:Analysis Class Incremental IID}

We demonstrate the accuracy of the model on the \textit{classes encountered so far} during the FL process as the classes increment. Specifically, in our setup, the accuracy for session 1 refers to the accuracy of the global target model in the communication round $T=20$ on a test set that contains only classes $\{0, 1\}$. The accuracy for session 2 refers to the accuracy of the global target model in the communication round $T=40$ on a test set that contains only classes $\{0, 1, 2, 3\}$, and so on. Session 5 represents the accuracy of the final global target model on the complete test set, which is the final accuracy shown in Table \ref{Table Main Result}. This method is a popular way to illustrate Class Incremental CL scenarios, providing a fine-grained view of the model's learning performance at any given time. It is evident that the model cannot accurately classify classes it has never encountered, so showing the accuracy on unseen classes is not meaningful.

As shown in Tables \ref{Table Class Incremental IID 1} and \ref{Table Class Incremental IID 2}, the accuracy varies across three datasets in the Class Incremental scenarios. For session 1, most of the frameworks are essentially vanilla FedAvg, and accuracy differences arise solely from the uncertainties in the training process. In session 2, the CL strategies of all frameworks start to take effect to avoid catastrophic forgetting. FedAvg, FedProx, FedAvg+LwF, and FedAvg+EWC fail to remember the previous data distribution, as they can only correctly classify the current classes $\{2, 3\}$, resulting in a subsequent accuracy drop of about 50\%. By session 3, these frameworks can remember only the current 2 classes out of a total of 6 classes, leading to an accuracy of approximately $\frac{2}{6}$. In contrast, frameworks specifically designed for CFL problems perform significantly better, demonstrating their ability to overcome catastrophic forgetting to varying degrees. Comparatively, the accuracy declines faster for CIFAR-10, indicating that its images are rich in information and complex patterns, making them more prone to forgetting. Therefore, for the CIFAR-10 dataset, more samples per client and additional training epochs are necessary for the model to learn the representations of different classes better.

Note that we focus on the Class Incremental IID scenario because it only has a subset of classes at any given time or session, whereas Class Incremental Non-IID includes all classes. Therefore, the Class Incremental Non-IID scenario does not require fine-grained analysis of the accuracy on currently encountered classes and can be directly tested on the complete test set.

\begin{table}[h]
\small
\caption{\textbf{Analysis of Scenario - Class Incremental IID.} Tested on encountered classes.}
\label{Table Class Incremental IID 1}
\centering
\begin{tabular}{@{}lccccc||ccccc@{}}
\toprule
\multirow{2}{*}{\textbf{Method}} & \multicolumn{5}{c||}{\textbf{MNIST - Session}} & \multicolumn{5}{c}{\textbf{Fashion-MNIST - Session}} \\
\cmidrule(r){2-6} \cmidrule(l){7-11}
& 1 & 2 & 3 & 4 & 5 & 1 & 2 & 3 & 4 & 5 \\
\midrule
FedAvg \cite{mcmahan2017communication} & \textbf{100.00} & 49.07 & 31.07 & 24.90 & 19.77 & 99.40 & 49.00 & 33.33 & 25.00 & 19.96 \\
FedProx \cite{li2020federated} & \textbf{100.00} & 49.03 & 31.07 & 24.77 & 19.78 & 99.45 & 49.10 & 33.33 & 25.00 & 19.96 \\
FedAvg+LwF \cite{li2017learning} & 99.95 & 49.07 & 31.07 & 24.88 & 19.77 & 99.40 & 49.05 & 33.33 & 25.00 & 19.96 \\
FedAvg+EWC \cite{kirkpatrick2017overcoming} & 99.95 & 49.07 & 31.32 & 26.46 & 19.78 & \textbf{99.55} & 49.27 & 33.35 & 25.00 & 19.95 \\
FedAvg+ACGAN \cite{odena2017conditional} & \textbf{100.00} & 87.76 & 61.37 & 52.90 & 42.15 & 99.25 & 82.28 & 60.87 & 37.11 & 40.83 \\
FedCIL \cite{qi2023better} & 99.84 & 81.78 & 73.38 & 66.99 & 46.36 & 98.57 & 66.51 & 54.59 & 35.95 & 38.17 \\
FOT \cite{bakman2023federated} & \textbf{100.00} & 47.53 & 38.27 & 32.76 & 32.18 & 99.45 & 48.58 & 48.35 & 34.61 & 27.47 \\
MFCL \cite{babakniya2024data} & \textbf{100.00} & 49.07 & 33.43 & 26.86 & 20.55 & 99.45 & 48.95 & 33.33 & 25.00 & 19.96 \\
TARGET \cite{zhang2023target} & \textbf{100.00} & 50.88 & 35.07 & 26.38 & 21.15 & 99.40 & 49.65 & 33.10 & 24.77 & 19.81 \\
\midrule
\NAME (Ours) & 99.95 & \textbf{99.62} & \textbf{98.86} & \textbf{97.52} & \textbf{94.69} & \textbf{99.55} & \textbf{90.40} & \textbf{80.98} & \textbf{68.28} & \textbf{72.50} \\
\bottomrule
\end{tabular}
\end{table}

\begin{table}[h]
\centering
\small
\caption{\textbf{Analysis of Scenario - Class Incremental IID.} Tested on encountered classes.}
\label{Table Class Incremental IID 2}
\begin{tabular}{@{}lccccc}
\toprule
\multirow{2}{*}{\textbf{Method}} & \multicolumn{5}{c}{\textbf{CIFAR-10 - Session}} \\
\cmidrule{2-6}
& 1 & 2 & 3 & 4 & 5 \\
\midrule
FedAvg \cite{mcmahan2017communication} & 94.55 & 41.70 & 29.32 & 24.01 & 18.74 \\
FedProx \cite{li2020federated} & \textbf{95.45} & 41.40 & 29.37 & 24.02 & 18.60 \\
FedAvg+LwF \cite{li2017learning} & 94.45 & 41.75 & 29.70 & 24.07 & 18.67 \\
FedAvg+EWC \cite{kirkpatrick2017overcoming} & 94.40 & 41.17 & 29.48 & 24.16 & 18.64 \\
FedAvg+ACGAN \cite{odena2017conditional} & 93.90 & 71.17 & 43.58 & 30.34 & 24.52 \\
FedCIL \cite{qi2023better} & 80.66 & 48.58 & 35.88 & 29.36 & 21.59 \\
FOT \cite{bakman2023federated} & 94.60 & 34.42 & 23.07 & 18.06 & 13.47 \\
MFCL \cite{babakniya2024data} & 94.80 & 41.05 & 28.97 & 23.95 & 18.62 \\
TARGET \cite{zhang2023target} & 94.80 & 47.10 & 31.32 & 23.43 & 18.74 \\
\midrule
\NAME (Ours) & 94.55 & \textbf{76.70} & \textbf{54.68} & \textbf{40.75} & \textbf{38.13} \\
\bottomrule
\end{tabular}
\end{table}

\subsection{Domain Incremental}
\label{appendix:Analysis Domain Incremental}

We show the accuracy of the model on the \textit{domains encountered so far} during the FL process as the domains increment. Unlike in the Class Incremental IID scenario, where the model needs a thorough understanding of different classes for accurate classification, in the Domain Incremental scenario, the model can potentially classify categories even from unseen domains. For example, a model trained on the Cartoon domain might correctly classify images from the Photo domain. Figure \ref{Fig. Domain Incremental} shows the changes in the model accuracy in the complete test dataset, where it can be observed that the model accuracy initially increases and then decreases. This is partly due to the nature of the dataset, where some domains (e.g., Art Painting) are not closely related to others. Additionally, the order of domains can lead to inconsistent accuracy variations. In this paper, we follow the setting from other literature, specifically Sketch $\rightarrow$ Cartoon $\rightarrow$ Art Painting $\rightarrow$ Photo (increasing the level of realism over time) \cite{zhao2023does}. If we consider decreasing the level of realism over time, the accuracy variations would present a different case.

For the above reasons, we should also consider presenting and analyzing the model's performance in the Domain Incremental scenario by testing only on the currently encountered domain, as shown in Table \ref{Table Domain Incremental}. It is evident that for our \NAME framework, there is a significant drop in accuracy from session 2 to session 3. This indicates a substantial difference between the synthetic data generated for session 1 (Sketch) and session 2 (Cartoon) compared to session 3 (Art Painting), leading to confusion in the target model on these data. In contrast, the accuracy drop for baselines occurs sharply from session 1 to session 2, and then slows down in subsequent domain changes. This further confirms the significant differences between domains. However, our \NAME framework can mitigate the impact of these domain differences on the performance of the target model.

\begin{minipage}[h]{0.45\textwidth}
\centering
\includegraphics[width=1\linewidth]{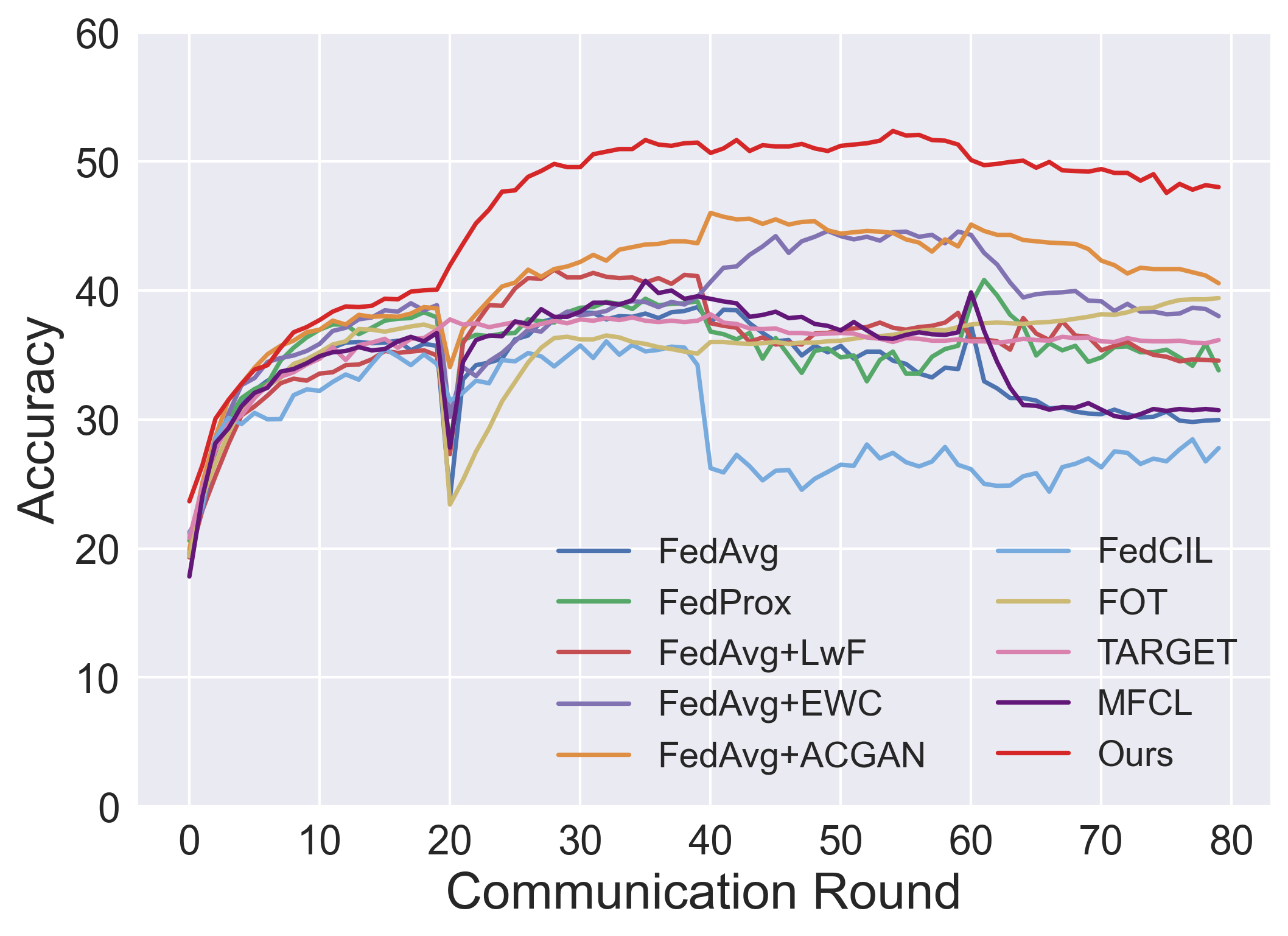}
\captionof{figure}{\textbf{Main Result - Domain Incremental Scenario.} Tested on the complete test set.}
\label{Fig. Domain Incremental}
\end{minipage}
\hfill
\begin{minipage}[h]{0.54\textwidth}
\centering
\scriptsize
\captionof{table}{\textbf{Analysis of Scenario - Domain Incremental.} Tested on encountered classes.}
\label{Table Domain Incremental}
\begin{tabular}{lccccc}
\toprule
\multirow{2}{*}{\textbf{Method}} & \multicolumn{4}{c}{\textbf{PACS - Domain}} \\
\cmidrule{2-5}
& 1 & 2 & 3 & 4 \\
\midrule
FedAvg \cite{mcmahan2017communication} & 63.36 & 45.82 & 33.21 & 29.96 \\
FedProx \cite{li2020federated} & 66.79 & 44.38 & 35.50 & 33.82 \\
FedAvg+LwF \cite{li2017learning} & 64.76 & 46.61 & 36.58 & 34.57 \\
FedAvg+EWC \cite{kirkpatrick2017overcoming} & 66.16 & 44.70 & 42.52 & 38.02 \\
FedAvg+ACGAN \cite{odena2017conditional} & 64.89 & 51.08 & 42.82 & 40.57 \\
FedCIL \cite{qi2023better} & 50.51 & 25.95 & 15.90 & 13.49 \\
FOT \cite{bakman2023federated} & 65.52 & 43.82 & 40.12 & 39.42 \\
MFCL \cite{babakniya2024data} & 66.79 & 53.34 & 41.25 & 36.16 \\
TARGET \cite{zhang2023target} & 65.02 & 44.83 & 35.91 & 30.72 \\
\midrule
\NAME (Ours) & \textbf{69.59} & \textbf{64.70} & \textbf{52.55} & \textbf{48.02} \\
\bottomrule
\end{tabular}
\end{minipage}

\section{Sensitivity Study}
\label{appendix:Sensitivity}

\subsection{Effect of Number of Clients (Sample Size per Client)}
\label{appendix:Sensitivity NumClients}

For simulation datasets such as MNIST, Fashion-MNIST, and CIFAR-10, where increasing the number of clients results in a reduction in the data available to each client. In our setup, with a total of 5 sessions, when the number of clients is set to 20, the sample size per client aligns with that set in FedAvg. However, variations in client numbers change the sample size per client, affecting not only the training of the target model but also the training of the diffusion model. Consequently, we analyze the effect of different numbers of clients on accuracy, as illustrated in Figure \ref{Fig. Sensitivity Study - NumClients}. For synthetic datasets consistent with the main text, we ensure that the sample sizes of the synthetic datasets match those of the real datasets.

The results demonstrate that the \NAME framework performs better when there are fewer clients and each client has a larger sample size, which means the data are more concentrated. This improvement is rationalized by the fact that the diffusion model has access to more training data, resulting in more realistic and less noisy synthetic data. Moreover, since our setup ensures that the synthetic datasets have the same number of samples as the real datasets, fewer clients allow the diffusion model to generate more synthetic data, which in turn aids the training of the target model.

\begin{figure}[h]
    \centering
    \begin{subfigure}[h]{0.49\textwidth}
        \centering
        \includegraphics[width=\textwidth]{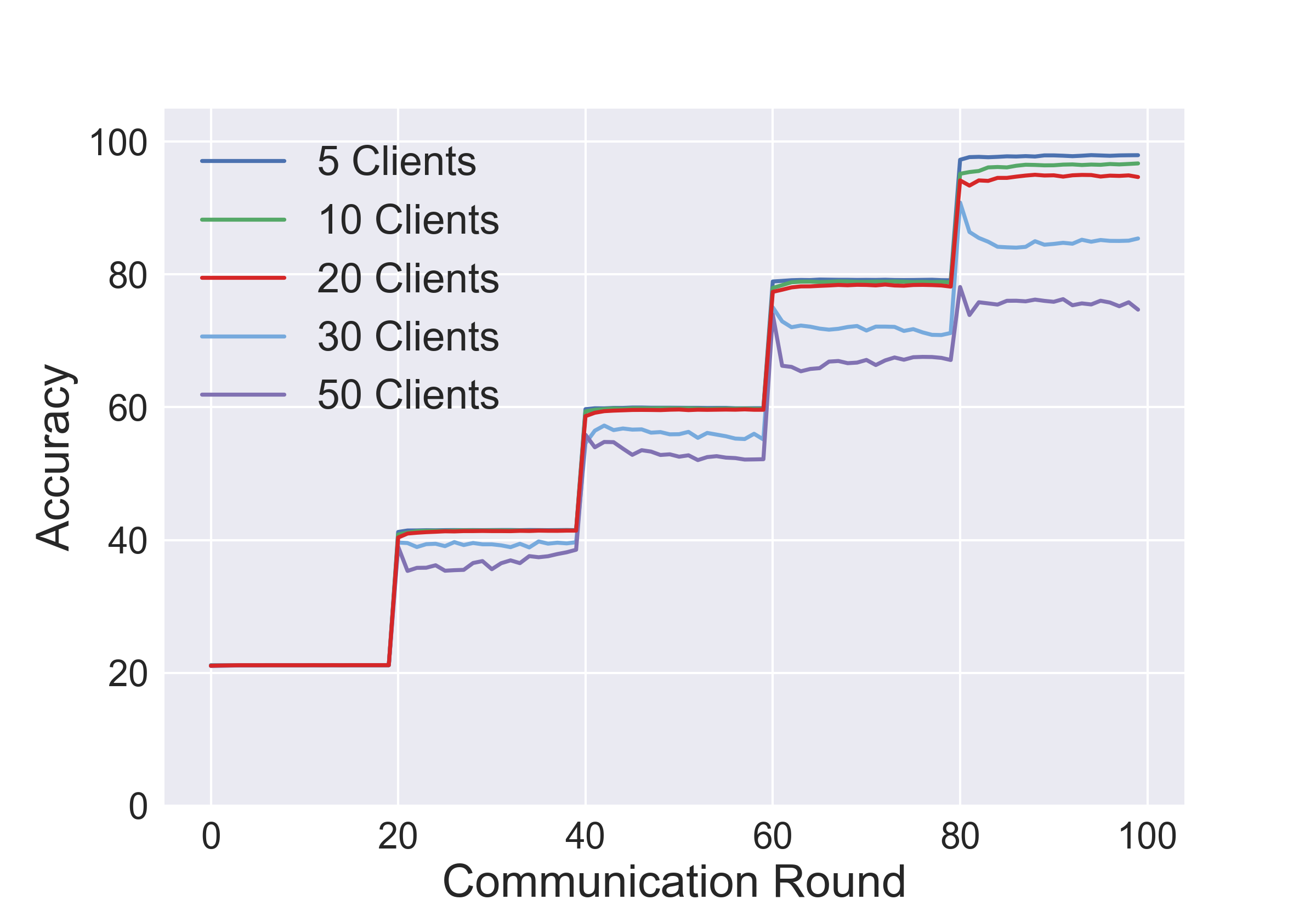} 
        \caption{Accuracy on the Global Test Set}
    \end{subfigure}
    \hfill
    \begin{subfigure}[h]{0.49\textwidth}
        \centering
        \includegraphics[width=\textwidth]{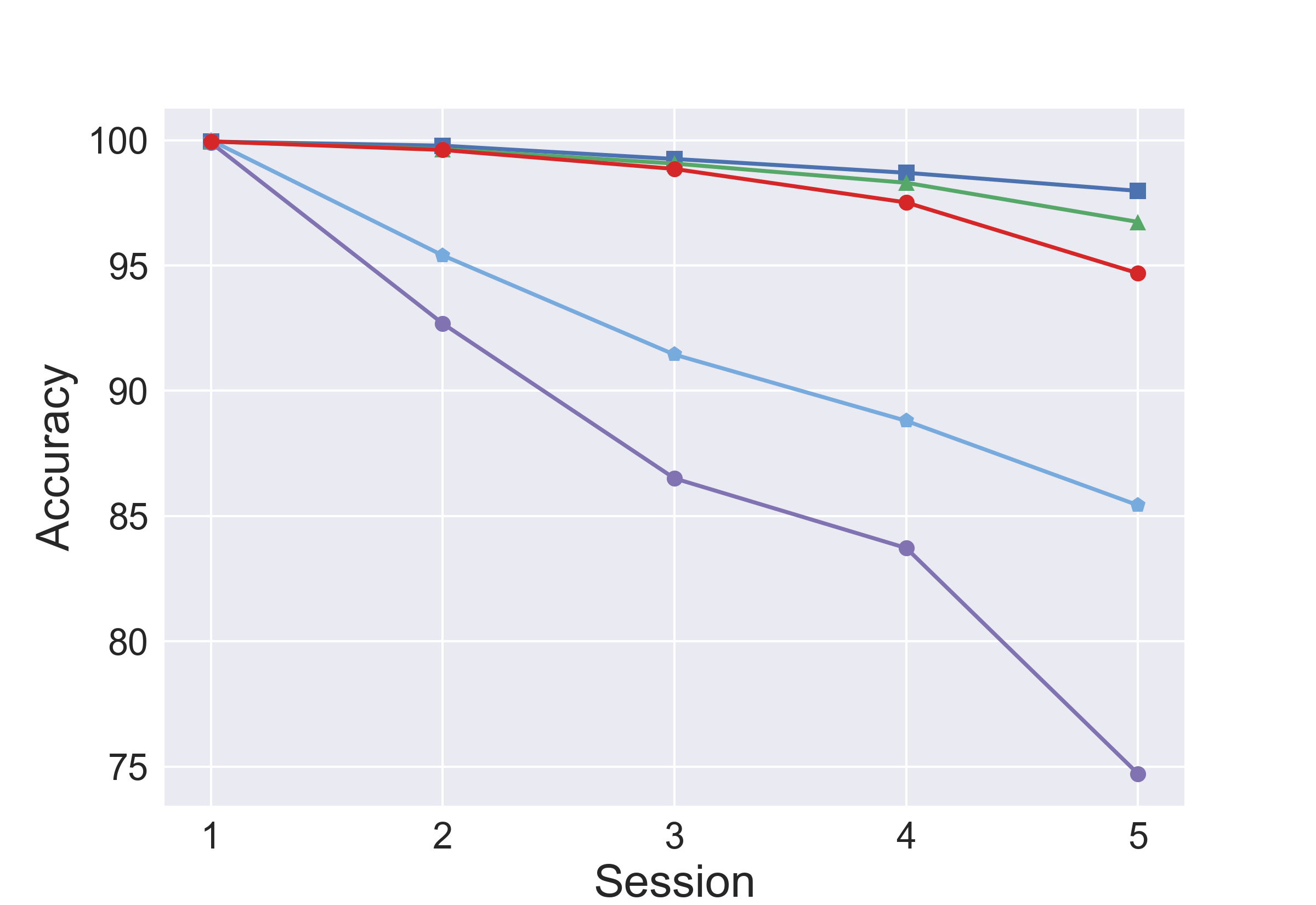}
        \caption{Accuracy Across Session Test Sets}
    \end{subfigure}
    \caption{\textbf{Sensitivity Study - Number of Clients.} Using the MNIST dataset in the Class-Incremental IID scenario as an example: 5 clients - 2400 samples per session per client, 10 clients - 1200 samples per session per client, 20 clients - 600 samples per session per client, 30 clients - 400 samples per session per client, 50 clients - 240 samples per session per client. The setup with 20 clients corresponds to the configuration in the main text.}
    \label{Fig. Sensitivity Study - NumClients}
\end{figure}

\subsection{Effect of Number of Synthetic Samples}
\label{appendix:Sensitivity NumSynSamples}

The number of synthetic samples generated needs to be balanced in real-world scenarios to maximize performance while minimizing computational overhead and memory usage. In the main text, we always set the number of synthetic samples equal to the number of real samples (i.e., $\delta=1$) to ensure a balance between previous and current knowledge during target model training. This varies across different datasets; for example, in the MNIST dataset, the number of synthetic samples is 600, while in the CIFAR-10 dataset, it is 500. Although the diffusion model can generate an unlimited number of synthetic samples, it is evident that when there are too many synthetic samples, the target model may struggle to learn effective current knowledge. We analyze the effect of different number of synthetic samples on accuracy, as illustrated in Figure \ref{Fig. Sensitivity Study - NumSynSamples}.

The results show that performance is optimal when the number of synthetic samples equals the number of real samples, i.e., with the ratio $\delta=1$. Fewer synthetic samples cannot provide sufficient training quality for the target model, while an excess of synthetic samples causes the model to overly focus on previous knowledge at the expense of current knowledge. Additionally, the diffusion model's training dataset comprises the real data from the current session and the synthetic data from previous sessions. Therefore, if there are too many synthetic samples from previous sessions, they will continue to influence the diffusion model's training in the current session, preventing the diffusion model from adequately learning the real data of the current session.

\begin{figure}[h]
    \centering
    \begin{subfigure}[h]{0.49\textwidth}
        \centering
        \includegraphics[width=\textwidth]{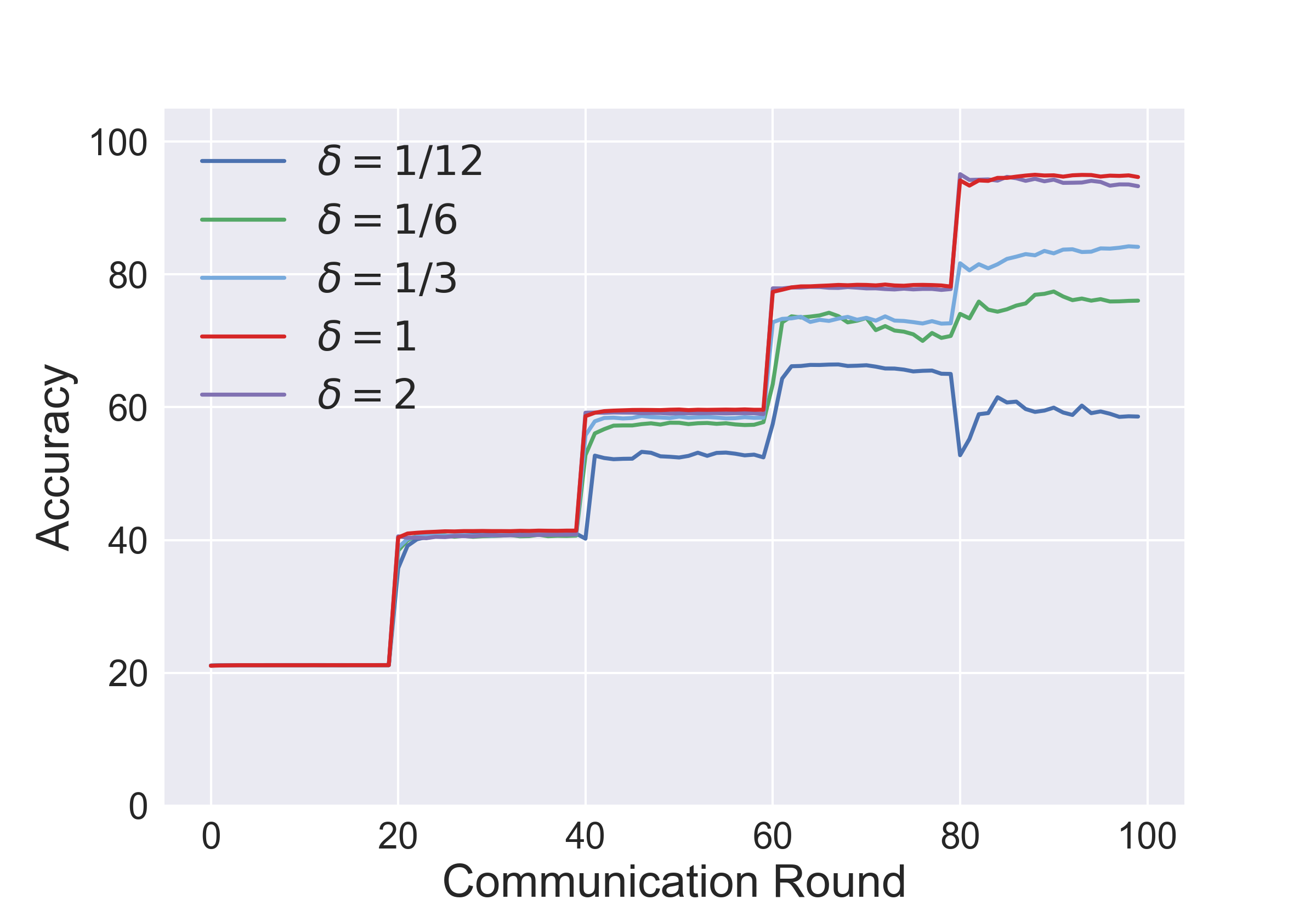} 
        \caption{Accuracy on the Global Test Set}
    \end{subfigure}
    \hfill
    \begin{subfigure}[h]{0.49\textwidth}
        \centering
        \includegraphics[width=\textwidth]{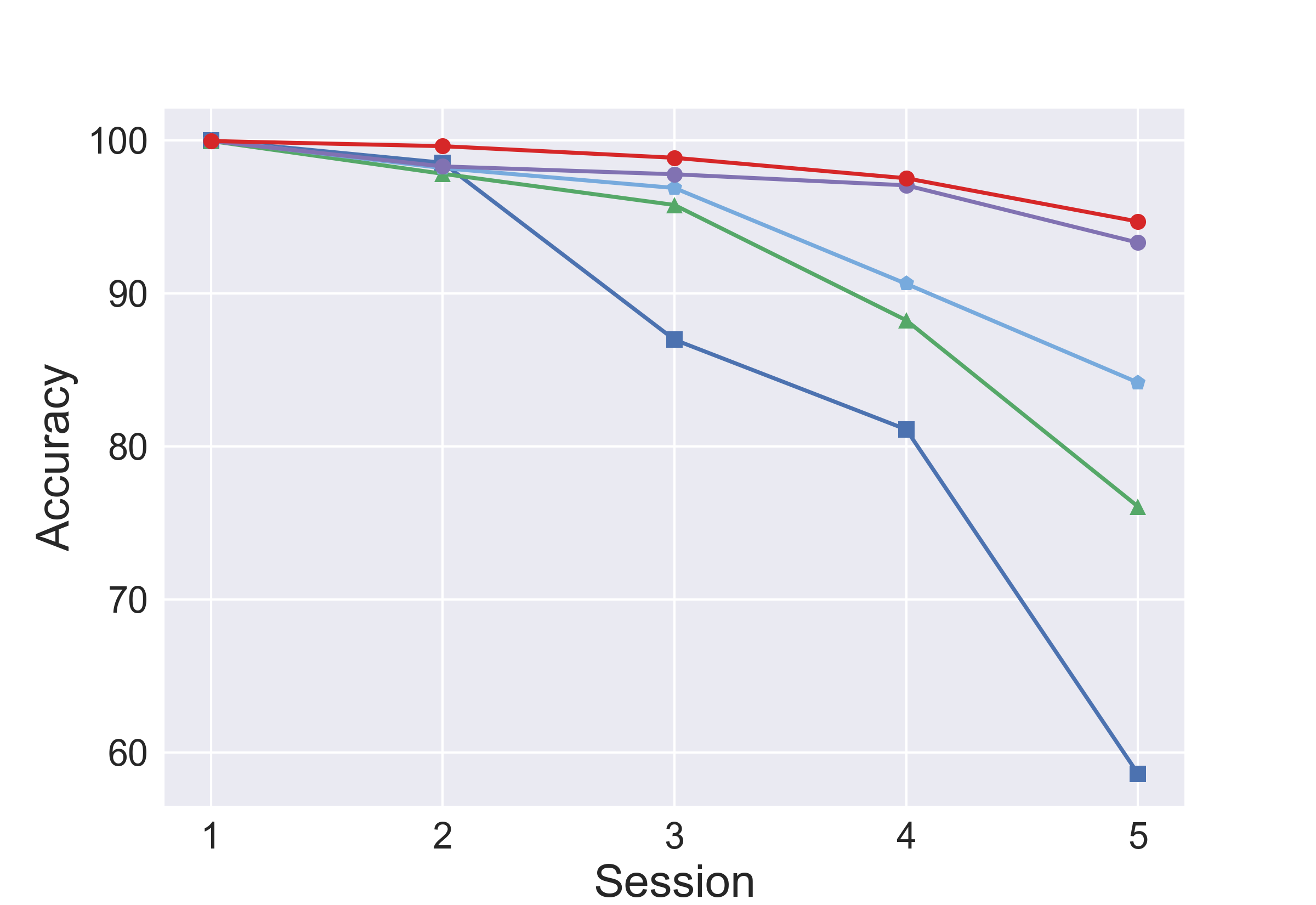}
        \caption{Accuracy Across Session Test Sets}
    \end{subfigure}
    \caption{\textbf{Sensitivity Study - Number of Synthetic Samples.} Using the MNIST dataset in the Class-Incremental IID scenario as an example. We maintain a setup of 20 clients and 600 real samples per session per client while varying the number of synthetic samples generated. The setup with $\delta=1$ corresponds to the configuration in the main text.}
    \label{Fig. Sensitivity Study - NumSynSamples}
\end{figure}

\newpage

\section{Broader Impacts}
\label{appendix:Broader Impacts}

\NAME is a CFL framework designed to address the issue of catastrophic forgetting in dynamic FL scenarios. It has a wide range of societal impacts, promoting applications and deployments across multiple fields, but it also carries potential risks. Besides the typical privacy and security concerns faced in FL, including a generative model also inherits potential issues associated with generative models. Below, we outline the positive impacts of \NAME, its potential risks, and mitigation strategies.

\textbf{Potential Positive Societal Impacts.}

\begin{itemize}[leftmargin=*]
    \item \textbf{Clients Operate in Dynamic Environment.} \NAME effectively supports multiple clients in ever-changing environments. The targets may change continuously, such as personnel, vehicles, buildings, etc., or the environment itself may vary, such as day and night, seasonal changes, different weather conditions, etc. For example, in an intelligent transportation scenario, a static roadside unit captures photos of vehicles in various weather conditions and needs to remember the features of previous weather conditions while continuing to learn in the current weather.
    \item \textbf{Clients Move Through Different Environments.} \NAME effectively supports multiple dynamic clients performing tasks in different environments because it prevents clients from forgetting knowledge from previous environments. This is crucial as the targets in these different environments are likely to be similar. For example, an unmanned aerial vehicle (UAV) patrolling different environments, from towns to highways to forests, may encounter similar manifestations of potential hazards like fires.
\end{itemize}

\textbf{Potential Negative Societal Impacts.}

\begin{itemize}[leftmargin=*]
    \item \textbf{Malicious Attacks.} Although clients in \NAME do not send the diffusion model to anyone, including the server or other clients, which minimizes the risk of privacy leakage, there is still a potential risk of replaying sensitive private data. If an attacker gains access to the diffusion model on a client, they could potentially replay all historical data. In contrast, if a client in vanilla FedAvg is compromised, only the current data are at risk. Given the broader temporal span and richness of the data that could be exposed, the former scenario is clearly more severe. Therefore, researchers should consider implementing privacy protection techniques for generative models when deploying \NAME to prevent attackers from extracting sensitive data from synthetic samples.
    \item \textbf{Misuse.} The diffusion model in \NAME also carries the risk of misuse. For example, some users might use sensitive data to train the diffusion model to circumvent regulatory scrutiny of stored data. Therefore, regulatory bodies need to implement stringent oversight and verification processes to ensure that generative models are not being used to bypass data compliance regulations.
\end{itemize}

\end{document}